\documentclass{article}
\usepackage{amsmath,amsfonts}
\usepackage{array}
\usepackage{textcomp}
\usepackage{stfloats}
\usepackage{url}
\usepackage{verbatim}
\usepackage{graphicx}
\usepackage{cite}
\usepackage{tabularx}
\usepackage{tcolorbox}
\usepackage{array}  % for vertical centering in tables
\usepackage{subcaption}
 \usepackage{amsmath}
\usepackage[utf8]{inputenc} % allow utf-8 input
\usepackage{xcolor}
\usepackage{multirow}
\usepackage{xcolor}
\usepackage{comment}

\definecolor{Federated}{RGB}{228,26,28}   % red
\definecolor{Independent}{RGB}{55,126,184} % blue
\definecolor{Random}{RGB}{77,175,74}      % green
\definecolor{FTS}{RGB}{152,78,163}        % purple

\usepackage{graphicx}
\definecolor{burntorange}{RGB}{230,159,0}
\definecolor{softorange}{RGB}{251,193,94}
\definecolor{cobaltblue}{RGB}{0,114,178}
\definecolor{charcoal}{RGB}{0,0,0}
\definecolor{darkgreen}{RGB}{0,158,115}
\definecolor{mintgreen}{RGB}{102,221,170}
\definecolor{magenta}{RGB}{204,121,167}
\definecolor{peach}{RGB}{244,165,193}
\definecolor{brickred}{RGB}{213,94,0}
\definecolor{slategray}{RGB}{153,153,153}
\definecolor{3dblue}{RGB}{0,0,139}
\definecolor{3dgreen}{RGB}{0,100,0}
\definecolor{3dgray}{RGB}{128,128,128}
\definecolor{orange}{RGB}{255,127,0}       % orange

\usepackage{amsmath}

\newcommand{\solidline}[2]{\textcolor{#1}{\rule[0.5ex]{#2}{0.8pt}}}

\usepackage[T1]{fontenc}    % use 8-bit T1 fonts
\usepackage{hyperref}       % hyperlinks
\usepackage{url}            % simple URL typesetting
\usepackage{booktabs}       % professional-quality tables
\usepackage{amsfonts}       % blackboard math symbols
\usepackage{nicefrac}    
\usepackage{hyperref}
\usepackage{microtype}      % microtypography
\usepackage{xcolor}         % colors
\usepackage{enumitem} 
\usepackage{graphicx} 
\usepackage{amssymb}
\usepackage{comment}
\usepackage{amsthm}

\newtheorem{theorem}{Theorem}

\newtheorem{lemma}{Lemma}

\theoremstyle{definition}

\newtheorem{assumption}{Assumption}

\usepackage{microtype}
\usepackage{graphicx}
\usepackage{subcaption}

\usepackage{hyperref}

\usepackage[preprint]{icml2026}

\usepackage{amsmath}
\usepackage{amssymb}
\usepackage{mathtools}
\usepackage{amsthm}

\usepackage[capitalize,noabbrev]{cleveref}

\theoremstyle{plain}
\theoremstyle{definition}
\theoremstyle{remark}

\usepackage[textsize=tiny]{todonotes}

\icmltitlerunning{Collaborative Contextual Bayesian Optimization}

\begin{document}

\twocolumn[
  \icmltitle{Collaborative Contextual Bayesian Optimization}



  \icmlsetsymbol{equal}{*}

\title{Collaborative Contextual Bayesian Optimization}

  \begin{icmlauthorlist}
    \icmlauthor{Chih-Yu~Chang}{1}
    \icmlauthor{Qiyuan Chen}{2}
    \icmlauthor{Tianhan~Gao}{3}
    \icmlauthor{David~Fenning}{4}
    \icmlauthor{Chinedum~Okwudire}{3}
    \icmlauthor{Neil~Dasgupta}{3}
    \icmlauthor{Wei~Lu}{3}
    \icmlauthor{Raed~Al~Kontar}{2}
  \end{icmlauthorlist}

  \icmlaffiliation{1}{Department of Mathematics, Imperial College London, London, United Kingdom}
  \icmlaffiliation{3}{Department of Mechanical Engineering, University of Michigan, Ann Arbor, MI, USA.}
  \icmlaffiliation{2}{Department of Industrial \& Operations Engineering, University of Michigan, Ann Arbor, MI, USA.}
    \icmlaffiliation{4}{Department of Chemical and Nano Engineering, University of California, San Diego, CA, USA}
  \icmlaffiliation{2}{Department of Industrial \& Operations Engineering, University of Michigan, Ann Arbor, MI, USA.}
  \icmlaffiliation{2}{Department of Industrial \& Operations Engineering, University of Michigan, Ann Arbor, MI, USA.}

  \icmlcorrespondingauthor{Raed Al Kontar}{alkontar@umich.edu}

  \icmlkeywords{Collaborative learning, Personalization, contextual Bayesian optimization, optimal design, manufacturing}

  \vskip 0.3in
]



\printAffiliationsAndNotice{}  % no special notice (required even if empty)

\begin{abstract}
Discovering optimal designs through sequential data collection is essential in many real-world applications. While Bayesian Optimization (BO) has achieved remarkable success in this setting, growing attention has recently turned to context-specific optimal design, formalized as Contextual Bayesian Optimization (CBO). Unlike BO, CBO is inherently more challenging as it must approximate an entire mapping from the context space to its corresponding optimal design, requiring simultaneous exploration across contexts and exploitation within each. In many modern applications, such tasks arise across multiple potentially heterogeneous but related clients, where collaboration can significantly improve learning efficiency. We propose \texttt{CCBO}, Collaborative Contextual Bayesian Optimization, a unified framework enabling multiple clients to jointly perform CBO with controllable contexts, supporting both online collaboration and offline initialization from peers' historical beliefs, with an optional privacy-preserving communication mechanism. We establish sublinear regret guarantees and demonstrate, through extensive simulations and a real-world hot rolling application, that \texttt{CCBO} achieves substantial improvements over existing approaches even under client heterogeneity. The code to reproduce the results can be found at \href{https://github.com/cchihyu/Collaborative-Contextual-Bayesian-Optimization}{this URL}.
\end{abstract}

Identifying an optimal design that yields the best possible response is critical in many applications, where the relationship between the design and the response is typically unknown. This problem is known as a black-box optimization task \cite{frazier2018tutorial} where the search for the optimal response is naturally framed as solving a maximization problem. More specifically, the goal is to identify the optimal design $x^*=\arg\max_x ~ f(x)$, where $f$ is the response function and $x$ represents the design variable(s).

Bayesian optimization (BO) has been established, both theoretically and empirically, as an effective and efficient framework for black-box optimization under limited experimental budgets. BO leverages historical observations to construct a probabilistic surrogate model, which encodes the current belief about the underlying response function. An acquisition function (AF), defined on top of the surrogate, then guides the sequential search by balancing exploration and exploitation. At each iteration, the design that maximizes the AF is selected, its response is observed, and this new data point is incorporated into the historical dataset to update the surrogate model until the evaluation budget is exhausted. For recent surveys on BO and its advances, see \cite{frazier2018tutorial,shahriari2015taking}.

In many real-world applications, the goal is not to identify a single optimal design, but to find context-specific optimal designs. For instance, in medical treatment, the optimal drug dosage (the design) may depend on a patient’s age or weight (the context). Similarly, in manufacturing, the optimal operating parameters may change with target product specifications. The BO counterpart of such problems has been recently referred to as Contextual BO (CBO). More specifically, the response function in CBO is modeled as $f(x,c)$, where $x$ and $c$ denote the design and context variables, respectively. At each iteration, a context-design pair is selected for evaluation (i.e., to observe its response). Here, it is worth noting that in some settings, the context is \textit{controllable} and can be deliberately chosen by the experimenter, while in others it may be observed externally and cannot be controlled. For example, in additive manufacturing, one may optimize printing parameters such as speed or laser power while treating the target part thickness as the context; in this case, the experimenter can choose/control which thickness to test at each iteration. In this paper, we focus on the controllable-context setting, where the learner is free to select which contexts to sample in order to efficiently approximate the context-specific optimal design $x^*(c)=\arg\max_x f(x,c)$ by sequentially collecting context-design-response observations. This setting is often referred to as offline CBO (OCBO) \cite{char2019offline}, which forms an important branch of CBO.

Given this formulation, BO is a special case of CBO in which the context space collapses to a single point \cite{auer2002using, abe2003reinforcement}. From a practical perspective, a central benefit of CBO is that it can act as a \textit{recommendation system}. In particular, once the context-specific optimal design function $x^*(c)$ is learned, it can serve as a recommender: for any observed context $c$ (e.g., a target part thickness), one can immediately obtain the optimal design variables (printing speed and laser power) by evaluating $x^*(c)$, without restarting the experimentation process.

That said, while practically relevant, CBO is substantially more resource-consuming than BO, as it introduces a new challenge: estimating an entire function that maps contexts to their corresponding optimal designs, instead of identifying a single optimal point. This requires exploration across the context space, in addition to the usual exploration-exploitation tradeoff within each context. Fortunately, many modern applications involve multiple related clients or experimental units, where collaboration can significantly improve learning efficiency. For example, similar manufacturing machines operating at different locations may seek to collaborate so that context-specific optimal designs can be learned more efficiently by borrowing strength from each other.

With the rise of collaborative and distributed learning paradigms such as federated learning \cite{yue2024federated,yue2023gifair} bringing collaboration into optimal design has become more important than ever. Motivated by this opportunity, we propose Collaborative Contextual BO (\texttt{CCBO}), a unified framework that enables multiple clients to jointly perform CBO through cross-client collaboration. To the best of our knowledge, this is the first collaborative approach explicitly designed to enhance the efficiency of CBO. The main contributions of this work are summarized as follows:

\begin{itemize}
    \item \textbf{Collaborative CBO framework}: 
    We introduce a unified framework for CBO that enables multiple, \textit{potentially heterogeneous}, clients to jointly learn the mapping $x^*(c)$ by sharing information across related tasks, improving sample efficiency compared to independent optimization. The framework also supports an offline collaborative initialization mode, where a single client benefits from peers' historical knowledge without requiring coordinated experimentation.

    \item \textbf{Disagreement-driven switching}: 
    We propose a novel decision mechanism that identifies contexts where local and global recommendations differ and uses this disagreement to guide information sharing. This is coupled with an adaptive switching strategy that balances collaborative and client-specific decisions over time.
    
    \item \textbf{Theoretical guarantees}: 
    We establish sublinear regret for the proposed method, showing that it consistently improves its approximation of context-specific optimal designs as more data are collected.
\end{itemize}

%Furthermore, we show that the algorithm satisfies both the handover property and the no-harm guarantee first introduced by \cite{xu2024principled}.

The remainder of this paper is organized as follows. In Sec. \ref{sec:related}, we review the related literature. In Sec. \ref{sec:fcbobig}, we provide the mathematical formulation of the problem. The proposed algorithm is introduced in Sec. \ref{sec:FCBO}. Its theoretical properties are established in Sec. \ref{sec:theres} and \ref{sec:qualify}. Sec. \ref{sec:sim} presents extensive evaluations of the proposed algorithm on multiple benchmarks under varied settings. A real-world application to metal manufacturing is then demonstrated in Sec. \ref{sec:app}. Finally, we conclude the paper in Sec. \ref{sec:con}, where we also discuss limitations and directions for future work.

\section{Related Works}\label{sec:related}
As highlighted earlier, the key distinction between BO and CBO is that CBO requires the learner to approximate a context-dependent optimal design function $x^*(c)$ rather than identifying a single global optimum. Existing work on CBO can be broadly categorized based on whether the design and/or context spaces are treated as discrete or continuous. The discrete case is often studied under the \emph{contextual bandit problem} framework and has been investigated for nearly two decades \cite{abe2003reinforcement,auer2002using}, while the continuous case corresponds to the CBO setting, which has gained significant attention more recently \cite{krause2011contextual,char2019offline}.

In contextual bandits, the design space (often referred to as the action space in this literature) is assumed to be discrete, while the context space may be discrete or continuous. The response function is modeled jointly over contexts and actions, and the goal is to learn a policy that selects the best action for each observed context. For example, \cite{chu2011contextual} studied contextual bandits under linear response models, while \cite{slivkins2011contextual} treated the context as similarity information across different actions. This literature has also extended popular AFs in BO to the contextual bandit setting, including the Upper Confidence Bound (UCB) \cite{xu2020upper} and Expected Improvement \cite{gupta2022expected}.

\paragraph{Recent Advances in CBO} More recently, researchers have investigated the case where both the context and design spaces are continuous, which corresponds to the classical CBO setting. When the context is uncontrollable, once the client receives the context at each iteration, the goal is to select the best design that achieves the highest response under that given context. \cite{krause2011contextual} proposed one of the first algorithms that employed Gaussian processes ($\mathcal{GP}$) to model the response function jointly over context and design, and developed a UCB-based algorithm to identify context-specific designs. More recently, \cite{park2020contextual} proposed a trust-region-based algorithm for CBO that improves optimization efficiency in continuous spaces.

Another research direction investigates the case where the context is controllable, which is relatively new and more challenging than the uncontrollable setting. In this setting, the learner is free to choose which contexts to sample, and the objective is to approximate the context-specific optimal design function by sequentially collecting context--design--response observations. This problem was first studied by \cite{char2019offline}, who proposed a multi-task Thompson sampling (MTS) approach. They emphasized the importance of two key tasks that are critical for success in this setting: (i) exploration of the context space as a high-level task, since the ultimate goal is to identify the optimal design across all contexts, and (ii) selection of the design given the chosen context, which requires balancing exploration and exploitation. More recently, a two-level approach was proposed by \cite{le2025controller}, which employs two $\mathcal{GP}$ models: one to map the context space to the corresponding optimal design, and another to model the response function jointly over the context and design spaces. These works highlight the unique challenges of controllable-context CBO, where learning an accurate context-to-optimal-design mapping requires efficient exploration across contexts in addition to standard design optimization.

%In addition, \cite{zhan2025collaborative} introduced a collaborative scheme in which clients share surrogate models under communication constraints. Collectively, these works highlight the need for collaborative BO in settings where data privacy must be preserved.

\paragraph{Applications in CBO} Compared to BO, the development of CBO is still in its infancy. Nevertheless, its successful applications have highlighted the importance of this problem setting. For instance, CBO has been employed in personalized treatment allocation to improve cancer therapy strategies by adaptively selecting effective interventions based on patient-specific contexts \cite{durand2018contextual}. Similarly, \cite{cereda2021cgptuner} applied CBO to database management system autotuning, where the optimal configurations depend on workload conditions and IT stack parameters. More recently, \cite{liu2024contextual} leveraged CBO for congestion pricing, efficiently designing distance-based toll schemes by incorporating temporal contextual information from day-to-day traffic dynamics. These examples illustrate the versatility of CBO in optimizing decision-making under diverse contextual conditions.

\paragraph{Collaborative BO} Recent advances in federated and distributed learning, together with the increasing computational capabilities of edge devices, have motivated a new research direction: enabling multiple clients to collaboratively search for optimal designs while preserving their own historical observations. The key motivation is that \textit{collaboration can accelerate the trial-and-error process}, allowing clients to reach high-performing designs faster than if they were optimizing independently. Collaborative BO was first studied by \cite{dai2020federated, dai2021differentially}, who proposed federated Thompson sampling (FTS) to distribute BO efforts across clients while preserving privacy. Later, \cite{yue2025collaborative} addressed the challenge of client heterogeneity, where each client aims to recover its individualized optimal design. They proposed consensus BO, in which each client selects its next design as a weighted sum of the design points suggested by all clients. %, with the weights gradually adjusted so that each client increasingly focuses on its own response patterns. 
More recently, \cite{chen2025multi} extended this line of work by proposing an automatic hedging mechanism: at each iteration, when a client receives design suggestions from other clients, it draws a finite number of posterior samples and discards those expected to yield worse performance. A recent survey on this topic can be found in \cite{al2024collaborative}. Overall, these works demonstrate that collaborative BO can significantly accelerate the discovery of optimal designs, both theoretically and in practice \cite{liu2024scalable,zhan2025collaborative}. However, these approaches focus on identifying a single optimal design per client. In contrast, our setting requires learning a context-dependent mapping $x^*(c)$, where collaboration must account for variation across the context space.

With this, we conclude by noting that, to the best of our knowledge, this is the first paper to bring CBO into a distributed and collaborative paradigm.

% However, when it comes to the collaborative contextual bandit problem, the literature remains limited. The only existing works focus on contextual bandits with linear payoffs, which correspond to the case where the context is not controllable \cite{huang2023federated, huang2021federated, zhou2023differentially}. The design of collaboration schemes for CBO, particularly in the setting with controllable contexts, remains an open problem and presents many opportunities for future research.

\section{Collaborative Contextual BO}\label{sec:fcbobig}

\subsection{Problem settings and surrogate modeling} 
\label{sec:Problem_settings}
We consider $K$ clients, each seeking their individual \textit{context-specific optimal design} by sequentially collecting context--design--response triplets. 
Let $\mathcal{X}$ denote the design space, $\mathcal{C}$ the context space, 
and for each client $k \in [K]$, let the response function be
\[
f_k : \mathcal{X} \times \mathcal{C} \;\to\; \mathbb{R}.
\]
The \textit{context-specific optimal design} for client $k$ under context $c \in \mathcal{C}$, denoted $x^{*}_k(c)$, is defined as
\[
x^{*}_k(c) \;=\; \underset{x \in \mathcal{X}}{\operatorname{argmax}} \; f_k(x, c).
\]

\begin{figure}[ht]
    \centering
        \begin{subfigure}[t]{0.25\textwidth}
        \includegraphics[width=\linewidth]{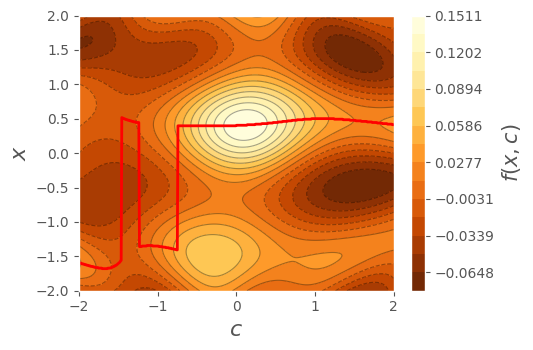}
        \caption{CBO}
    \end{subfigure}
    \hfill
    \begin{subfigure}[t]{0.21\textwidth}
        \includegraphics[width=\linewidth]{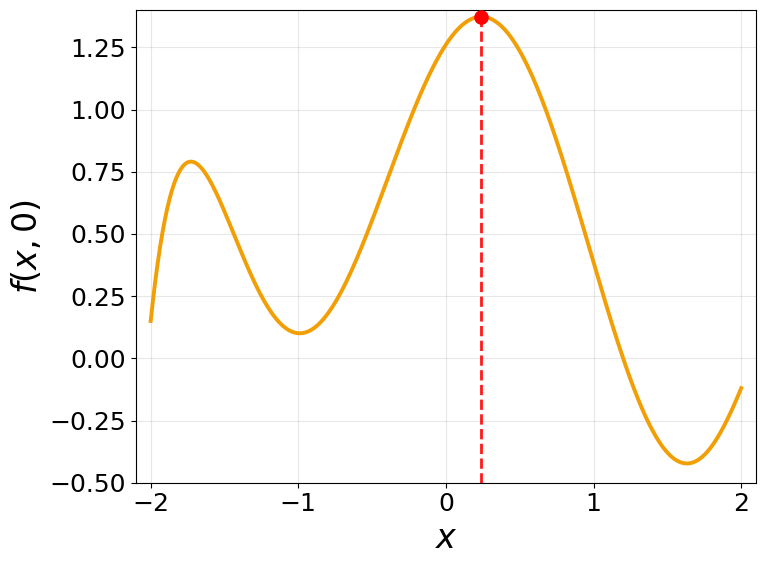}
        \caption{BO}
    \end{subfigure}

    \caption{Comparison of BO and CBO. \textbf{(a)}: The heat map shows the value of $f(x,c)$ and red curve shows $x^{*}_k(c) = \arg\max_{x \in \mathcal{X}} f_k(x, c)$. \textbf{(b)}: The orange curve shows the response function at $c=0$ and the red dot marks the single optimal design $x^*(0)=\arg\max_x f(x,0)$ (context fixed).}
    \label{fig:BOCBO}
\end{figure}
Fig.~\ref{fig:BOCBO} illustrates the conceptual difference between CBO and BO. In the left panel, the red curve represents the context-specific optimal design, which associates each context (x-axis) with the design (y-axis) that maximizes the function value $f(x,c)$. In contrast, the right panel shows the goal of BO, where the red dot indicates the design that maximizes $f(x)$ when the context is fixed. The function in the right panel is identical to that in the left panel but with $c=0$. In this sense, BO can be regarded as a special case of CBO, where the context space degenerates to a singleton.

To establish notation, we assume that each client is allocated a budget of $T$ experimental iterations. Prior to these iterations, each client $k$ is initialized with a dataset $\mathcal{D}_{k,0}$. At iteration $t \in [T]$, client $k$ selects a design–context pair $(x_{k,t}, c_{k,t})$ and observes a response
\[
y_{k,t} \;=\; f_k(x_{k,t}, c_{k,t}) + \epsilon_{k,t},
\]
where $\epsilon_{k,t}$ denotes observation noise with mean zero and variance $\sigma_\epsilon^2$. The resulting triplet $(x_{k,t}, c_{k,t}, y_{k,t})$ is then incorporated into the dataset via the update \(\mathcal{D}_{k,t} = \mathcal{D}_{k,t-1} \cup \{(x_{k,t}, c_{k,t}, y_{k,t})\} \).

At time $t$, given the historical observations $\mathcal{D}_{k,t}$, we fit a surrogate model to represent the posterior belief about the response function $f_k(x,c)$ conditioned on $\mathcal{D}_{k,t}$. A common choice is a $\mathcal{GP}$, which provides both predictions at unseen inputs and a corresponding measure of uncertainty. We denote this posterior by $\mathcal{GP}(\mathcal{D}_{k,t})$. A $\mathcal{GP}$ is specified by a kernel function $\mathcal{K}:(\mathcal{X}\times\mathcal{C})^2 \to \mathbb{R}$ and a prior mean function $\mu(\cdot)$, which is assumed to be zero. 

Conditioned on $\mathcal{D}_{k,t}$, the $\mathcal{GP}$ posterior yields two key quantities for any design-context pair $(x, c)$: the posterior mean $\mu_{k,t}(x, c)$ and the posterior variance $(\sigma_{k,t}(x, c))^2$. These quantities admit closed-form expressions:
\[
\begin{aligned}
\mu_{k,t}(x, c) &= \mathbf{k}_{k,t}(x, c)^\top \left( \mathbf{K}_{k,t} + \sigma_\epsilon^2 \mathbf{I} \right)^{-1} \mathbf{y}_k, \\
\sigma_{k,t}(x, c)^2 &= \mathcal{K}((x, c), (x, c)) \\
&\quad - \mathbf{k}_{k,t}(x, c)^\top \left( \mathbf{K}_{k,t} + \sigma_\epsilon^2 \mathbf{I} \right)^{-1} \mathbf{k}_{k,t}(x, c).
\end{aligned}
\]
Here, $\mathbf{k}_{k,t}(x, c)$ denotes the vector of kernel evaluations between $(x,c)$ and the observed input pairs in $\mathcal{D}_{k,t}$ under $\mathcal{K}$, $\mathbf{K}_{k,t}$ is the corresponding kernel matrix formed by evaluating $\mathcal{K}$ over all observed input pairs in $\mathcal{D}_{k,t}$, and $\mathbf{y}_k = (y_{k,1}, \ldots, y_{k,t})^\top$.

\subsection{\texttt{CCBO}: Collaborative Contextual Bayesian Optimization}\label{sec:FCBO}

While Fig.~\ref{fig:BOCBO} highlights that CBO is more complex than classical BO, the central question is how to design a mechanism that enables multiple clients to jointly identify their context-specific optimal designs. To this end, we propose a simple yet effective collaboration scheme. Our method does not require sharing raw observations $\mathcal{D}_{k,t}$ and can be implemented efficiently. Moreover, it can be extended to support privacy-preserving operation, as proposed in Sec.~\ref{sec:qualify} and empirically evaluated in Sec~\ref{sec: fed}.

\paragraph{Collaboration via disagreement-driven decision making}

The surrogate model plays a central role in CBO. Each client $k$ maintains a Gaussian process surrogate $\mathcal{GP}(\mathcal{D}_{k,t})$ based on its own observations. However, the model used to make the next sampling decision need not coincide exactly with this local surrogate. To distinguish these two roles, we refer to the model used for decision-making as the \emph{operational model}. This distinction is especially important in the collaborative setting. Early in the optimization process, each client typically has only a limited number of observations, so its local surrogate may be too uncertain to reliably identify the most informative context-design pair. To improve decision-making in this regime, we construct a collaborative operational model by aggregating information across clients.

Specifically, at iteration $t$, we define
\begin{equation}\label{eq:mubar}
\bar{\mu}_{t-1} = \frac{1}{K}\sum_{k=1}^K \mu_{k,t-1}
\end{equation}
as the average posterior mean across all clients. Unlike the local posterior mean $\mu_{k,t-1}$, which is based only on client $k$'s dataset, $\bar{\mu}_{t-1}$ pools information from all clients and serves as a shared operational model.

Using $\bar{\mu}_{t-1}$, we identify contexts where client $k$'s current belief may be most unreliable. To formalize this, let
\[
x_{k,t-1}^{*}(c) = \arg\max_x \mu_{k,t-1}(x,c),
\]
denote client $k$'s current locally optimal design at context $c$, which may be suboptimal due to limited or noisy local observations, and let
\[
\bar{x}_{t-1}^{*}(c) = \arg\max_x \bar{\mu}_{t-1}(x,c),
\]
denote the design recommended by the collaborative operational model. We define the \textit{collaborative improvement potential} for client $k$ at context $c$ as
\begin{equation}\label{eq:deltaB}
\Delta_{k,t-1}^{\text{collab}}(c)
:= \bar{\mu}_{t-1}\big(\bar{x}_{t-1}^*(c),c\big)
-\bar{\mu}_{t-1}\big(x_{k,t-1}^{*}(c),c\big).
\end{equation}

This quantity is always nonnegative by construction. Importantly, the improvement is evaluated under the collaborative model $\bar{\mu}_{t-1}$ rather than the local model $\mu_{k,t-1}$, as $\bar{\mu}_{t-1}$ serves as the operational model guiding context selection. A larger value of $\Delta_{k,t-1}^{\text{collab}}(c)$ indicates that, under this shared model, client $k$'s current design is suboptimal at context $c$, thereby identifying contexts where local and global recommendations disagree the most. This criterion can be interpreted as targeting contexts where the client's local estimate is most unreliable. At such contexts, the algorithm evaluates the design recommended by the collaborative model, thereby resolving the discrepancy and improving the local model.

Accordingly, under the collaborative scheme, client $k$ selects the next context as
\begin{equation}\label{eq:ctb}
c_{k,t}^{\text{collab}} = \arg\max_c \Delta_{k,t-1}^{\text{collab}}(c),
\end{equation}
and then selects the corresponding design as
\begin{equation}\label{eq:xtb}
x_{k,t}^{\text{collab}} = \bar{x}_{t-1}^*(c_{k,t}^{\text{collab}}).
\end{equation}
This collaborative rule is particularly effective when clients are similar and local datasets are still small, since pooling information improves the calibration of the operational model. In the homogeneous case, $\bar{\mu}_{t-1}$ is expected to approach the common response function faster than any individual local model. More generally, even under moderate heterogeneity, the shared model provides a useful guide in early stages.

However, collaboration is not uniformly beneficial. When client response functions differ, the shared model $\bar{\mu}_{t-1}$ reduces variance but introduces bias relative to the client-specific function $f_k$. As more local data are collected, the local surrogate $\mu_{k,t-1}$ may become more accurate than the global model. This motivates a mechanism that dynamically transitions from collaborative to client-specific decision-making.

\paragraph{Adaptive switching between collaborative and independent decisions}

We introduce a probability-based switching mechanism that dynamically balances collaborative and independent decision-making. At iteration $t$, each client selects between the two schemes through a Bernoulli gate with probability $p_t$. With probability $p_t$, the client follows the collaborative strategy to select the next design--context pair using \eqref{eq:ctb} and \eqref{eq:xtb}. Otherwise, with probability $1 - p_t$, the client makes decisions independently. Such probability-based switching mechanisms have been used in BO \cite{dai2021differentially} to reduce reliance on shared or approximate models. However, to the best of our knowledge, they have not been studied in the context of collaborative contextual Bayesian optimization.

Under the independent scheme, each client constructs its operational model by sampling from its own surrogate model, rather than relying on a global or aggregated model. Specifically, at time $t$, the operational model is defined as a posterior sample $\tilde{f}_{k,t-1}$ drawn from the client’s posterior $\mathcal{GP}(\mathcal{D}_{k,t-1})$. This follows a Thompson sampling (TS) approach, where decisions are made using a sampled function to balance local exploration and exploitation.

In practice, this sample is obtained by drawing function values over a finite candidate set. Specifically, let $z=(x,c)\in\mathbb{R}^D$ denote the concatenated design--context input, and define
\begin{equation}
\label{eq:S}
S := \{z^{(1)}, \dots, z^{(M)}\} \subseteq \mathcal{X} \times \mathcal{C}.
\end{equation}
This set is used to construct the operational model.

Given the sampled function $\tilde{f}_{k,t-1} \sim \mathcal{GP}(\mathcal{D}_{k,t-1})$, we define the context-specific optimal design as
\[
\tilde{x}_{k,t-1}^*(c) = \arg\max_{x \in \mathcal{X}} \tilde{f}_{k,t-1}(x, c),
\]
and define the corresponding \textit{local exploration potential} as
\begin{equation}\label{eq:delta}
\Delta_{k,t}(c)
= \tilde{f}_{k,t-1}\big(\tilde{x}^*_{k,t-1}(c), c\big)
- \tilde{f}_{k,t-1}\big(x_{k,t-1}^*(c), c\big).
\end{equation}

The next context is selected as
\begin{equation}\label{eq:ct}
c_{k,t} = \arg\max_{c \in \mathcal{C}} \Delta_{k,t}(c),
\end{equation}
and the corresponding design is
\begin{equation}\label{eq:xt}
x_{k,t} = \tilde{x}_{k,t-1}^*(c_{k,t}).
\end{equation}
This sampling-based procedure implicitly performs joint optimization over design-context pairs, in a similar fashion to MTS, while remaining entirely client-specific.

While both schemes define decisions through a notion of potential improvement, they serve fundamentally different purposes. The collaborative scheme targets contexts where local and global recommendations disagree, aiming to correct inaccurate local estimates using cross-client information sharing. In contrast, the independent scheme uses posterior sampling to promote exploration within a client’s own model, without leveraging cross-client information.

Taken together, the proposed framework addresses two complementary challenges: (i) correcting inaccurate local estimates due to limited information and (ii) reducing uncertainty in the local surrogate model. By gradually decreasing $p_t$ over time, the method transitions from disagreement-driven, cross-client learning to individualized refinement of the optimal design.
 
\begin{figure*}[htbp]
    \centering
    \includegraphics[width=0.95\textwidth]{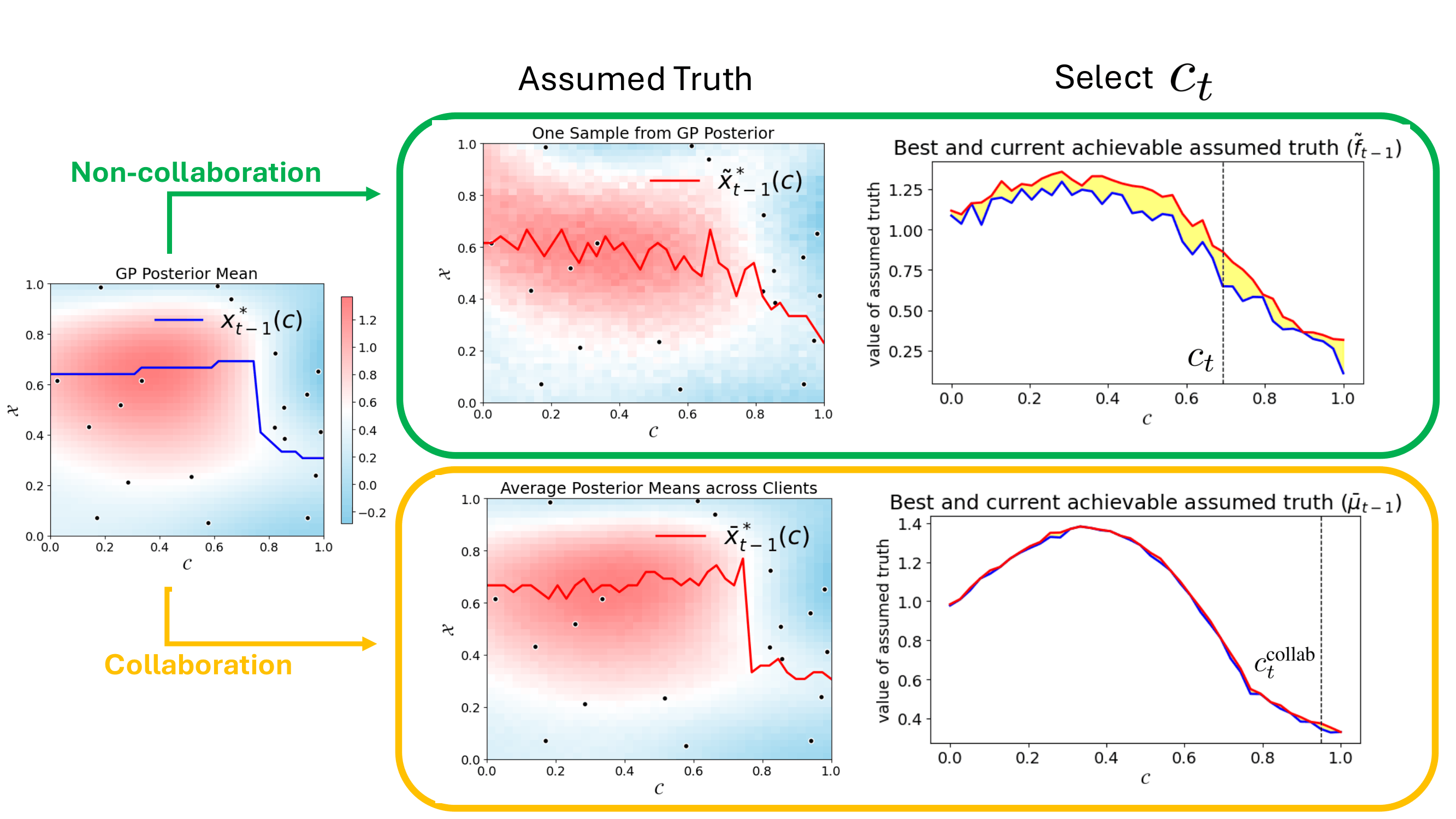}
    \caption{Graphical illustration of \texttt{CCBO}. The left, upper-middle, and lower-middle columns show the posterior mean, a posterior sample, and the averaged posterior mean across all clients, respectively. The colors represent the function values, and the curve connects the best design (shown on the $y$-axis) corresponding to each plot. The left column plots the relationship between the context and the response achieved by each method across all contexts. }
    \label{fig:posterior_comparison}
\end{figure*}

We summarize the proposed \texttt{CCBO} in Algorithm~\ref{alg:CCBO}. All optimization steps in Algorithm~\ref{alg:CCBO} are implemented over finite candidate sets as described above.

\paragraph{An illustrative example}

A graphical illustration of the proposed \texttt{CCBO} with $K=10$ is provided in Fig.~\ref{fig:posterior_comparison}. For clarity, we take the perspective of a single client $k$ and suppress the client subscript $k$ throughout this discussion.

The figure is organized into three columns, each corresponding to a key step in the decision-making process. The left column shows the current posterior belief: both the non-collaborative case (green box) and the collaborative case (orange box) share the same posterior belief $x_{t-1}^*(c)$ (blue curve) as their current estimate of the context-specific optimal design.

The two cases diverge in the second column, which represents \textit{what each client assumes to be the ground truth}. With probability $1-p_t$, the client 
adopts a posterior sample $\tilde{f}_{t-1}$ (upper-middle panel) as the assumed truth; with probability $p_t$, it instead uses the averaged posterior mean $\bar{\mu}_{t-1}$ (lower-middle panel). In either case, the red curve in the second column traces the context-specific optimal design induced by the assumed  truth, i.e., the design the client would choose at each context if the assumed truth were correct.

The third column then quantifies the \emph{discrepancy} between what the client could achieve under the assumed truth (red curve) and what it would actually obtain using its current belief $x_{t-1}^*(c)$ (blue curve). These gaps (colored in yellow) correspond to the local exploration potential $\Delta_{t-1}(c)$ (defined in~\eqref{eq:delta}) and the collaborative improvement potential $\Delta_{t-1}^{\text{collab}}(c)$ (defined in~\eqref{eq:deltaB}) in the upper and lower panels, respectively.

Finally, the client selects the next context as the one with the largest discrepancy (indicated by the vertical line in the third column), and evaluates the corresponding context-specific optimal design under the assumed truth as the next design. Intuitively, this selection targets the context where the current belief is most suboptimal, so that the new observation can yield the greatest potential improvement.

\paragraph{Offline collaborative initialization} A practically important special case of \texttt{CCBO} arises when other clients are unable to run new experiments, yet canshare their posterior means from historical observations. A single active client can then benefit from the collaborative branch of Algorithm~\ref{alg:CCBO} by treating $\bar{\mu}$ in~\eqref{eq:mubar} as \emph{fixed}, with each $\mu_{k}$ computed once from client $k$'s archived dataset. The disagreement-driven context selection in~\eqref{eq:ctb} and~\eqref{eq:xtb} then proceeds identically. This is particularly natural in manufacturing consortia, where peer machines have extensive historical records but are unavailable for coordinated experimentation. As the active client accumulates observations and $p_t$ decays, the algorithm transitions gracefully to fully independent operation via~\eqref{eq:ct} and~\eqref{eq:xt}.

\subsection{Privacy-Preserving Posterior Mean Sharing}\label{sec:qualify}

In the collaborative scheme, clients share their posterior means $\mu_{k,t-1}(x,c)$ to construct a global operational model $\bar{\mu}_{t-1}(x,c)$. While this does not require sharing raw observations $\mathcal{D}_{k,t}$, the posterior mean itself is learned from local data and may still reveal sensitive information. When privacy is a concern, we propose a mechanism that enables clients to communicate their posterior means in a compressed and less revealing form.

Specifically, we adopt a random Fourier features (RFF)-based representation \cite{rahimi2007random, dai2020federated} to approximate the posterior mean. Let $z=(x,c)\in\mathbb{R}^D$ denote the concatenated design-context input. The key idea is to approximate the kernel-induced function space using a randomized feature map, so that the posterior mean can be expressed approximately as a linear model in a fixed feature space.

Concretely, we draw random frequencies $\{\alpha_j\}_{j=1}^J$ from the spectral density of the kernel and random phases $\{\beta_j\}_{j=1}^J \sim \mathrm{Unif}[0,2\pi]$, and define the feature map
\[
\phi(z) = \sqrt{\tfrac{2}{J}}\big[\cos(\alpha_1^\top z + \beta_1), \ldots, \cos(\alpha_J^\top z + \beta_J)\big]^\top.
\]

Under this, the posterior mean is approximated as
\[
\mu_{k,t-1}(z) \approx \phi(z)^\top w_{k,t-1},
\]
where $w_{k,t-1} \in \mathbb{R}^J$ is a client-specific coefficient vector computed locally. To obtain $w_{k,t-1}$, each client fits this linear model to its own posterior mean evaluated on the candidate set $S$ defined in \eqref{eq:S} (one can also choose a different set $S_{\text{RFF}}\in\mathcal{X}\times\mathcal{C}$ for this approximation). We define: 
\[
\mathbf{m}_k = \big(\mu_{k,t-1}(z^{(1)}), \dots, \mu_{k,t-1}(z^{(M)})\big)^\top,
\]
\[
\mathbf{\Phi}_k = \big[\phi(z^{(1)}), \dots, \phi(z^{(M)})\big]^\top.
\]

The coefficient vector $w_{k,t-1}$ is then obtained via ridge regression:
\[
w_{k,t-1}
=
\arg\min_{w}
\|\mathbf{\Phi}_k w - \mathbf{m}_k\|_2^2 + \lambda \|w\|_2^2,
\]
which admits the closed-form solution
\begin{equation}
\label{eq:w}
w_{k,t-1}
=
(\mathbf{\Phi}_k^\top \mathbf{\Phi}_k + \lambda \mathbf{I})^{-1} \mathbf{\Phi}_k^\top \mathbf{m}_k.
\end{equation}

With this representation, collaboration can be implemented through a simple share-weights protocol:

\begin{enumerate}
    \item A central coordinator samples $\{\alpha_j, \beta_j\}_{j=1}^J$ and a shared input set $S$, and broadcasts them to all clients. This defines a common feature map $\phi(\cdot)$.
    
    \item Each client evaluates its posterior mean $\mu_{k,t-1}$ on $S$, computes $w_{k,t-1}$ using \eqref{eq:w}, and sends only this vector to the coordinator.

    \item The coordinator aggregates the weights as
    \[
    \bar{w}_{t-1} = \frac{1}{K} \sum_{k=1}^K w_{k,t-1}, \qquad
    \bar{\mu}_{t-1}(z) = \phi(z)^\top \bar{w}_{t-1},
    \]
    and broadcasts $\bar{w}_{t-1}$ (or equivalently $\bar{\mu}_{t-1}$) back to the clients.
\end{enumerate}

This approach replaces sharing a function (or many evaluations of it) with sharing a length-$J$ vector $w_{k,t-1}$, reducing communication cost and limiting information exposure.

Importantly, $w_{k,t-1}$ is a compressed representation of the posterior mean in a randomized feature space. No raw observations $\mathcal{D}_{k,t}$, kernel matrices, or explicit function values are transmitted. Moreover, reconstructing client-specific data from $(w_{k,t-1}, \phi)$ is an ill-posed inverse problem, since multiple datasets can induce similar representations. As such, this mechanism provides a practical way to enable collaboration while mitigating privacy concerns.

\begin{algorithm}[h]
\caption{\texttt{CCBO}}
\label{alg:CCBO}
\begin{algorithmic}[1]
\REQUIRE $T$, $K$, $\{p_t\}_{t=2}^T$ with $p_t \to 0$, initial data $\{\mathcal{D}_{k,0}\}_{k=1}^{K}$
\STATE Initialize: $p_1 \gets 1$
\FOR{$t = 1$ to $T$}
    \STATE \textbf{Step 1: Posterior update for each client}
    \FOR{$k = 1$ to $K$}
        \STATE Client $k$ updates its posterior belief $\mathcal{GP}(\mathcal{D}_{k,t-1})$
        \STATE Sample a realization $\tilde{f}_{k,t-1} \sim \mathcal{GP}(\mathcal{D}_{k,t-1})$
        \STATE Client $k$ releases $\mu_{k,t-1}$ to all clients
    \ENDFOR

    \STATE \textbf{Step 2: Select the next context and design}
    \FOR{$k = 1$ to $K$}
       \STATE Sample $s_{k,t} \sim \mathrm{Uniform}(0,1)$
        \IF{$s_{k,t} < p_t$}
            \STATE Set $c_{k,t}$ and $x_{k,t}$ according to \eqref{eq:ctb} and \eqref{eq:xtb}
        \ELSE
           \STATE Set $c_{k,t}$ and $x_{k,t}$ according to \eqref{eq:ct} and \eqref{eq:xt}
        \ENDIF
        \STATE Observe  $y_{k,t}$
        \STATE $\mathcal{D}_{k,t} = \mathcal{D}_{k,t-1} \cup \{(x_{k,t}, c_{k,t}, y_{k,t})\}$
    \ENDFOR
\ENDFOR
\end{algorithmic}
\end{algorithm}

\section{Theoretical Results}\label{sec:theres}

We focus on the performance of approximating $x_k^*(c)$ for the $k$-th client, and omit the client index when it is clear from the context. In this section, we establish a sublinear upper bound on the cumulative regret, which implies that the average regret vanishes as the number of iterations increases.

Despite the practical success of TS in both BO and CBO, its theoretical understanding in continuous settings remains limited. While the behavior of TS in bandit settings has been well studied \cite{russo2016information,agrawal2012analysis,agrawal2013thompson}, extending these results to continuous domains (i.e., BO) is more challenging and typically requires additional assumptions. A common approach is to discretize the domain when deriving theoretical guarantees \cite{dai2020federated,char2019offline}. We adopt this standard approach and introduce the following assumption.

\begin{assumption}\label{assump1}
The context space $\mathcal{C}$ and the design space $\mathcal{X}$ are finite, with $|\mathcal{C}|$ and $|\mathcal{X}|$ denoting the number of candidates in each space, respectively.
\end{assumption}

Under Assumption~\ref{assump1}, and considering the corresponding discretized algorithm, we define the following \emph{discretized regret} (also adopted in prior CBO work, e.g.,~\cite{char2019offline}):
\begin{equation}
\label{eq:disc_regret}
r_{k,t} \;=\; 
\frac{\displaystyle \sum_{c \in \mathcal{C}} \Big( f_k\big(x^{\text{best}}_k(c), c\big) - f_k\big({x}^{\text{best}}_{k,t}(c), c\big) \Big)}
{\displaystyle \sum_{c \in \mathcal{C}} \max_{x_1, x_2 \in \mathcal{X}} \Big\{ f_k(x_1, c) - f_k(x_2, c) \Big\}}\,,
\end{equation}
where we let $\mathcal{X}_{k,t}(c)\triangleq\{x_t:(x_t,y_t,c)\in \mathcal{D}_{k,t}\}$ and
\[
{x}^{\text{best}}_{k,t}(c) =
\begin{cases}
\underset{x\in\mathcal{X}_{k,t}(c)}{\arg\max}\; f_k(x,c), & \text{if } \mathcal{X}_{k,t}(c) \neq \emptyset,\\
\underset{x\in\mathcal{X}}{\arg\min}\; f_k(x,c), & \text{otherwise}.
\end{cases}
\]

This corresponds to the best observed design at context $c$, with a worst-case fallback when no observations are available. Under this discretized setting, we can derive an upper bound on the expected cumulative regret, as stated in Theorem~\ref{regbd}.

\begin{theorem}[Sublinear expected cumulative regret, proved in Appendix~\ref{app:pf1}]\label{regbd}
Under Assumption~\ref{assump1}, and assuming that $\sum_{t=1}^\infty t p_t < \infty$, the expected cumulative regret of a client after performing \texttt{CCBO} for $T$ iterations, defined as $R_T = \sum_{t=1}^T r_t$, satisfies $\mathbb{E}[R_T] = \mathcal{O}(\sqrt{\gamma_T T}).$
\end{theorem}

The notation $\mathcal{O}(\cdot)$ characterizes the asymptotic growth rate up to a constant factor. Theorem~\ref{regbd} implies that the cumulative regret grows sublinearly in $T$ (up to the information-gain term $\gamma_T$). Consequently, the average regret satisfies
\[
\frac{\mathbb{E}[R_T]}{T} \to 0 \qquad \text{as } T \to \infty,
\]
indicating that the proposed method is no-regret in the long run. In other words, although the algorithm may incur nonzero regret during early exploration, its per-iteration performance approaches that of the context-specific optimal design as more data are collected.

The condition $\sum_{t=1}^\infty t p_t < \infty$ requires the exploration probability to decay sufficiently fast. One simple choice is to let $p_t = \mathcal{O}(t^{-(2+\epsilon)})$ for any $\epsilon > 0$, which guarantees convergence of the series and ensures that the theorem applies. In practice, however, we observe that the proposed method is robust to different choices of the decay schedule, including less aggressive decay such as $p_t = 1/\sqrt{t}$, which performs extremely well empirically.

\paragraph{Communication frequency}
In addition to the RFF approximation, \texttt{CCBO} naturally limits the frequency of communication as decisions become more independent over time. As established in Theorem~\ref{sharebd}, the frequency with which each client participates in posterior sharing decreases over time and vanishes asymptotically as $t \to \infty$.

\begin{theorem}[Sublinear posterior sharing for each client, proved in Appendix~\ref{app:thm2}] \label{sharebd}
Let $I_T$ denote the total number of bits communicated up to time $T$ by \texttt{CCBO} with RFF approximation. Suppose that $\sum_{t=1}^T p_t = \mathcal{O}(\sqrt{T})$. Then, for any $\delta > 0$, with probability at least $1 - \delta$, we have
\[
I_T = \mathcal{O}\left(DK\sqrt{T\log(1/\delta)}\right),
\]
without a central server, and
\[
I_T = \mathcal{O}\left(\min\left\{DK\sqrt{T\log(1/\delta)}, DT\right\}\right),
\]
with a central server.
\end{theorem}

A sufficient choice is to let $p_t$ decay on the order of $t^{-1/2}$, for example, $p_t = \min\{1, c\,t^{-1/2}\}$ for some constant $c>0$. More generally, any schedule satisfying $p_t = \mathcal{O}(t^{-\alpha})$ with $\alpha \ge 1/2$ ensures $\sum_{t=1}^T p_t = \mathcal{O}(\sqrt{T})$, up to logarithmic factors in the boundary case $\alpha = 1/2$.

The quantity $I_T$ represents the total communication cost incurred up to time $T$. Thus, Theorem~\ref{sharebd} shows that the overall communication overhead grows sublinearly with the time horizon, rather than linearly in $T$. Equivalently, the average communication cost per iteration, $I_T/T$, vanishes as $T \to \infty$.

In practice, communication can be limited to iterations in which the collaborative scheme is selected (i.e., when the Bernoulli variable with probability $p_t$ triggers collaboration), leading to the sublinear communication cost described above.

The distinction between the cases with and without a central server arises from the communication architecture. Without a server, when one client shares its posterior information, it must be transmitted across all $K$ clients, leading to a communication cost that scales with both $D$ and $K$. With a central server, each client sends a $D$-dimensional representation of its posterior (which can be further reduced using the approach in Section~\ref{sec:qualify}) to the server, resulting in a lower communication cost per round. Consequently, the centralized architecture yields lower communication overhead than direct peer-to-peer sharing.

\begin{figure*}[t]
    \centering
    \begin{subfigure}[t]{0.32\textwidth}
        \includegraphics[width=\linewidth]{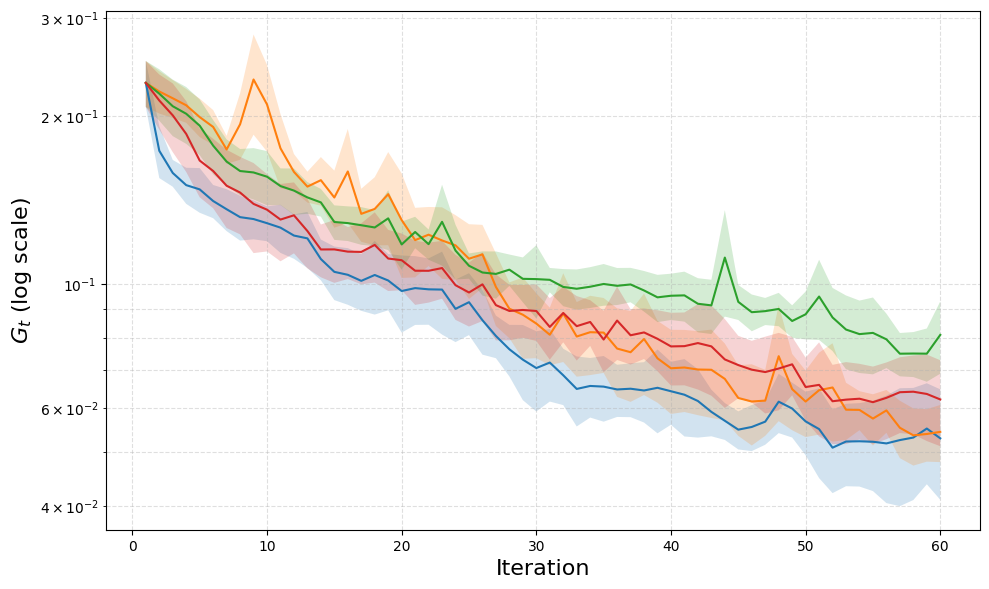}
        \caption{\texttt{ackley 2-1}}
    \end{subfigure}
    \hfill
    \begin{subfigure}[t]{0.32\textwidth}
        \includegraphics[width=\linewidth]{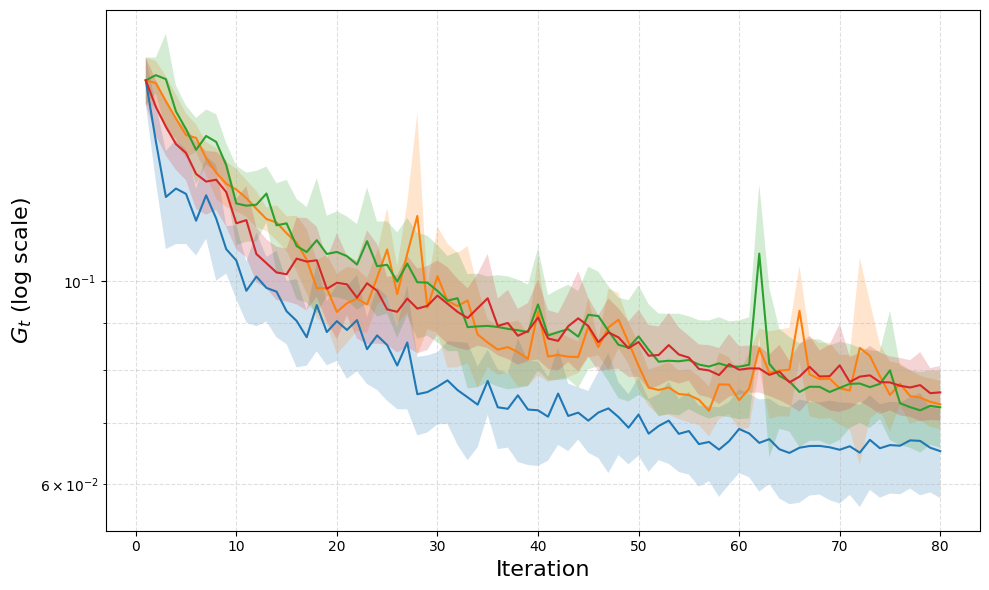}
        \caption{\texttt{ackley 2-2}}
    \end{subfigure}
    \hfill
    \begin{subfigure}[t]{0.32\textwidth}
        \includegraphics[width=\linewidth]{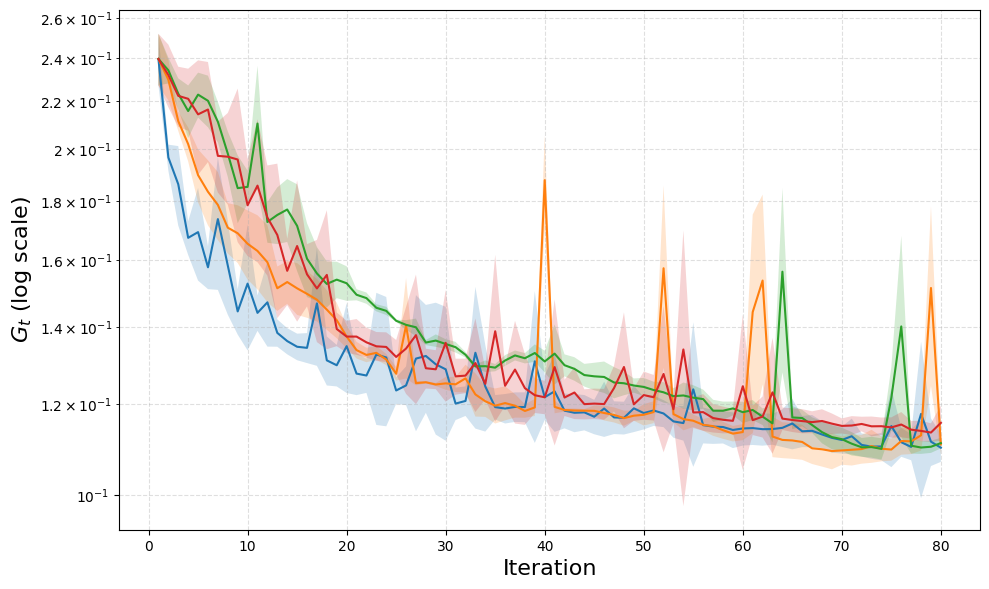}
        \caption{\texttt{ackley 1-3}}
    \end{subfigure}

    \caption{Comparison of regret $G_t$ in homogeneous settings using \texttt{ackley} function. 
    Each curve shows the mean regret with shaded 95\% confidence intervals. 
   \textbf{Methods}: 
    \solidline{Federated}{1.5em} FTS, 
    \solidline{Random}{1.5em} RS, 
    \solidline{orange}{1.5em} {MTS},
    \solidline{Independent}{1.5em} \texttt{CCBO}.}
    \label{fig:ackley}
\end{figure*}

\begin{figure*}[t]
    \centering
    \begin{subfigure}[t]{0.32\textwidth}
        \includegraphics[width=\linewidth]{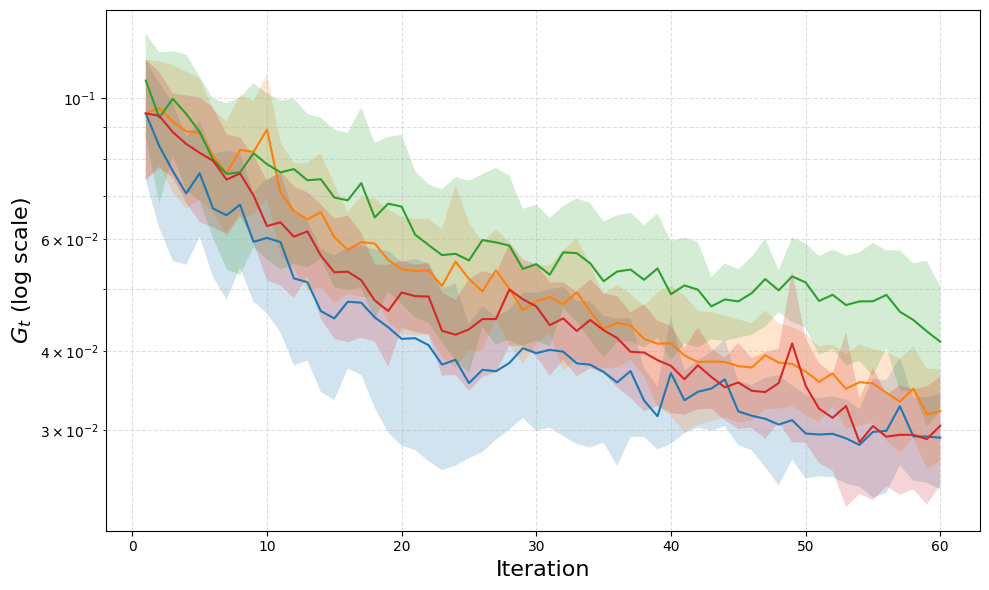}
        \caption{\texttt{levy 2-1}}
    \end{subfigure}
    \hfill
    \begin{subfigure}[t]{0.32\textwidth}
        \includegraphics[width=\linewidth]{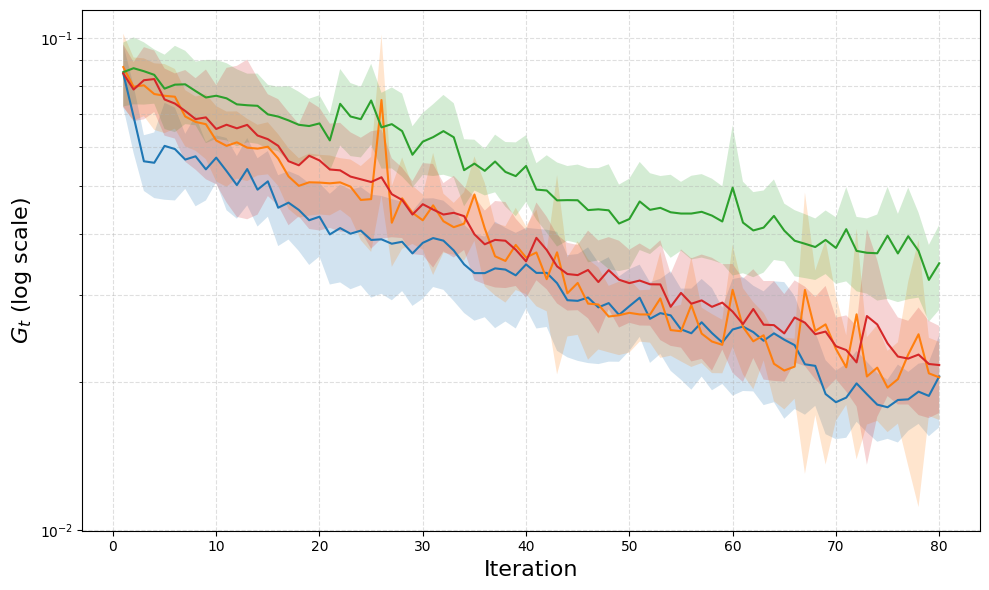}
        \caption{\texttt{levy 2-2}}
    \end{subfigure}
    \hfill
    \begin{subfigure}[t]{0.32\textwidth}
        \includegraphics[width=\linewidth]{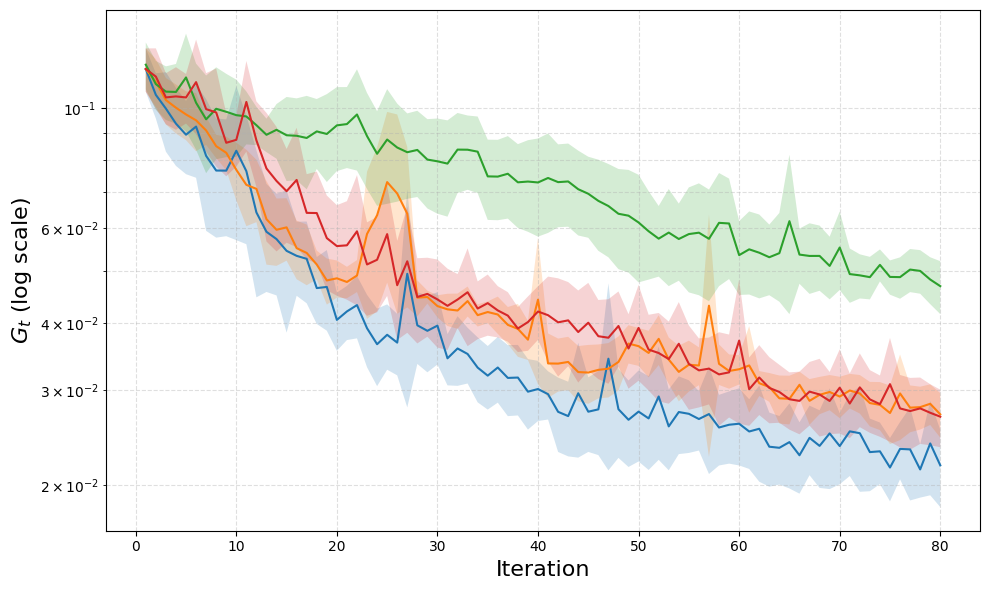}
        \caption{\texttt{levy 1-3}}
    \end{subfigure}

    \caption{Comparison of regret $G_t$ in heterogeneous settings using \texttt{levy} function. 
    Each curve shows the mean regret with shaded 95\% confidence intervals. 
    \textbf{Methods}: 
    \solidline{Federated}{1.5em} FTS, 
    \solidline{Random}{1.5em} RS, 
    \solidline{orange}{1.5em} {MTS},
    \solidline{Independent}{1.5em} \texttt{CCBO}
.}
    \label{fig:levy}
\end{figure*}

\section{Simulation Studies}\label{sec:sim}

In this section, we compare the performance of the proposed \texttt{CCBO} (Algorithm~\ref{alg:CCBO}) against existing baselines, including random sampling (RS) and independent multi-task TS (MTS) \cite{char2019offline}. We also include federated TS (FTS) \cite{dai2020federated}, where the next context is chosen at random in each iteration and FTS is then used to select the next design collaboratively. Let $D_x$ and $D_c$ denote the dimensions of the design space $x$ and the context space $c$, respectively. For numerical stability, we define both spaces on the unit hypercube, i.e., $[0,1]^{D_x}$ and $[0,1]^{D_c}$. The number of clients is set to 10. At each iteration, we randomly sample $100$ candidate designs and $100$ candidate contexts to form the discrete sets $\mathcal{X}$ and $\mathcal{C}$ used in that iteration for all clients, which is standard in CBO \cite{char2019offline}.

To evaluate the performance of each method over the continuous domains, we measure the log-scale overall simple regret. Specifically, $G_t$ is defined as the average, across all clients, of the normalized gap between the optimal design and the best design identified by the algorithm at iteration $t$, aggregated over the context space:
\begin{align*}
G_t \;=\; \frac{1}{K}\sum_{k=1}^K 
\frac{\int_{\mathcal{C}}\left(f_k(x_k^*(c),c)-f_k(x^*_{k,t}(c),c)\right)dc}
{\displaystyle \int_{\mathcal{C}}\left(\max_{x_1,x_2\in\mathcal{X}}  \left\{f_k(x_1,c)-f_k(x_2,c)\right\}\right)dc}.
\end{align*}
The normalization ensures that the regret is comparable across different functions and scales. 

Since evaluating these quantities exactly over continuous domains is intractable, we approximate $G_t$ by uniformly sampling $250$ context-design pairs $(x_i,c_i)$ from $\mathcal{X}\times\mathcal{C}$ and computing both the numerator and denominator using these samples. This resolution is empirically sufficient \cite{char2019offline}, and we adopt the same sampling scheme throughout all experimental evaluations. For each function with dimensions $D_x$ and $D_c$, we perform a total of $T=20(D_x+D_c)$ iterations, and the initial dataset for each client is set to $T_0=5(D_x+D_c)$.

All experiments are repeated 10 times to capture variability in the regret. Results are reported on a logarithmic scale to facilitate clear comparisons among methods. The benchmark functions considered include \texttt{ackley 2-1, 2-2, 1-3}, \texttt{levy 2-1, 2-2, 1-3}, and \texttt{hartmann 2-2}, all drawn from \cite{simulationlib}, where the suffix $D_c$--$D_x$ specifies the dimensionality of the context and design spaces, respectively. The details of the benchmark functions can be found in Appendix~\ref{appsim}. For all functions, we multiply the response by a negative sign to formulate them as maximization problems. The input space is normalized to the unit hypercube (i.e., $[0,1]^{D_x+D_c}$) to ensure fairness when considering heterogeneous cases.

For each benchmark function, we generate $1{,}000$ noiseless function evaluations at randomly sampled inputs from $\mathcal{X} \times \mathcal{C}$. We compute the empirical standard deviation of these values and use it to define the noise level. Specifically, the observed response is given by $y = f(z) + \epsilon$, where $\epsilon \sim \mathcal{N}(0, \sigma_\epsilon^2)$ and $\sigma_\epsilon = 0.1\,\hat{\sigma}_f$, with $\hat{\sigma}_f$ denoting the empirical standard deviation of the $1{,}000$ noiseless evaluations. This corresponds to a moderate noise level relative to the function variation. Throughout this section, we set $p_t = 1/\sqrt{t}$, which provides a balanced exploration-exploitation trade-off in practice.

\subsection{Homogeneous Settings}
We begin by demonstrating the effectiveness of the proposed method in the homogeneous case, where the response function is identical across all clients. The \texttt{ackley 2-1}, \texttt{ackley 2-2}, and \texttt{ackley 1-3} functions are used as the underlying response functions. The results of $G_t$ for all methods are shown in Fig.~\ref{fig:ackley}, from which several insights can be drawn. First, our method consistently outperforms all baselines across all settings. Second, compared to existing methods such as FTS, our approach substantially reduces regret in the early iterations, highlighting its ability to efficiently identify promising contexts and their corresponding optimal designs. This improvement is primarily due to effective information sharing across clients in early iterations, which accelerates the identification of high-quality regions of the design space.

\subsection{Heterogeneous Settings}
Next, we consider a more challenging scenario in which clients have similar yet distinct response functions. For each base function $f$, we define client $k$’s response function as
\begin{equation}\label{eq:hetero}
    f_k(x,c) = f(x + \xi^x_k, \, c + \xi^c_k),
\end{equation}
where $\xi^x_k$ and $\xi^c_k$ are vectors of length $D_x$ and $D_c$, respectively, with each element sampled from a uniform distribution between $-0.05$ and $0.05$. This introduces a random shift of up to 10\% in each dimension of the context and design spaces, since both spaces have been rescaled to the $[0,1]$ interval. The regret plots for the \texttt{levy 2-1}, \texttt{levy 2-2}, and \texttt{levy 1-3} functions are shown in Fig.~\ref{fig:levy}. It is noticeable that all methods outperform RS more significantly compared with the \texttt{ackley} setting. While MTS and FTS exhibit similar performance across most settings, our method consistently and significantly outperforms both benchmarks.

\begin{figure}[htbp]
    \centering
    \includegraphics[width=0.48\textwidth]{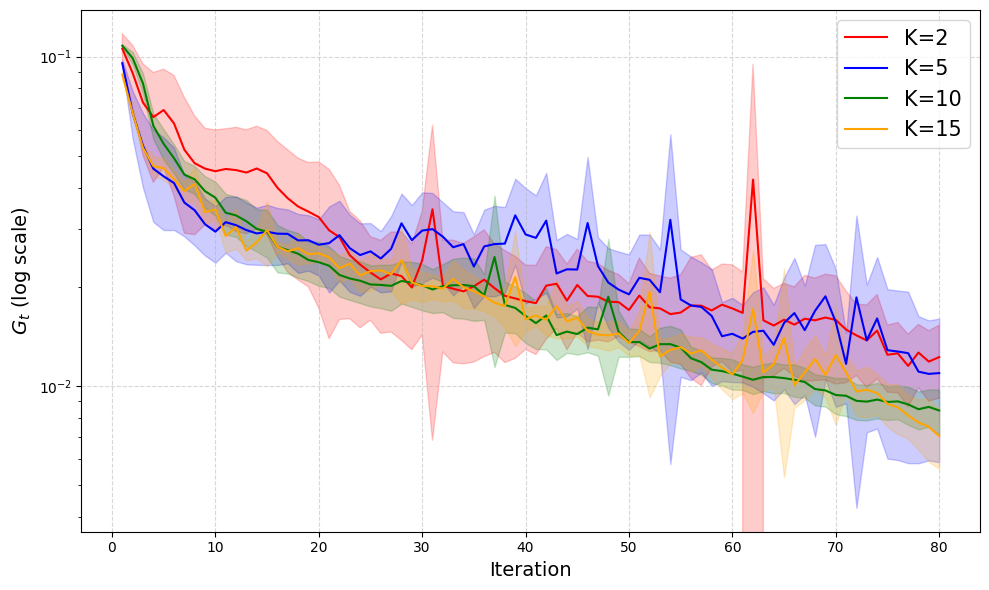}
    \caption{Comparison of regret $G_t$ with different number of clients using \texttt{hartmann 2-2} function in heterogeneous setting. Each curve shows the mean regret with shaded 95\% confidence intervals. }
    \label{fig:collaboration_comparison}
\end{figure}

\subsection{The Power of Collaboration}
To examine the benefit of collaboration, we employ \texttt{hartmann 2-2} as the ground-truth response function and generate heterogeneous response functions for each client according to (\ref{eq:hetero}). Different numbers of clients are used, with $K=2,\ 5,\ 10,\ 15$, while all other settings are kept the same as in Section~\ref{sec:sim}. The regret curves are shown in Fig.~\ref{fig:collaboration_comparison}. 

Several insights can be drawn. When the number of clients is small ($K=2$ and $5$), the performance is worse and the variability is higher. This is especially evident for $K=2$, where the ability to reduce regret is limited from the early iterations. In contrast, when the number of clients increases to 10, the performance improves substantially and becomes more stable. It is also noteworthy that while the final performance for $K=15$ is slightly better than for $K=10$, the two regret curves exhibit a similar trend.

\subsection{Cross Client Sharing with RFF Approximation}
\label{sec: fed}
As mentioned in Section~\ref{sec:qualify}, sharing the posterior mean can be achieved by transmitting the coefficients of its RFF approximation. This approach makes the sharing process more efficient and privacy-preserving, as it avoids the need to share the entire exact posterior mean. In this section, we evaluate the performance of transmitting the RFF coefficients and compare it against two baselines: sharing the posterior mean directly (i.e., \texttt{CCBO}) and the independent MTS method, which represents the current state-of-the-art.

We utilize the \texttt{Levy 1--3} functions under a heterogeneous setting with $J=50$ (the dimension of RFF basis), and all other configurations, including the number of repetitions, evaluations, $p_t$, $K$, $T$, and $T_0$, remain the same as those defined previously. The logarithmic regret plot is shown in Fig.~\ref{rffplot}. It is evident that our method, whether sharing the exact or approximate posterior across clients, consistently outperforms the independent MTS. Although sharing RFF coefficients may slightly degrade performance compared with \texttt{CCBO}, as $p_t$ decreases over time, both methods achieve comparable performance in the later stages. This highlights the advantage of extending \texttt{CCBO} through RFF approximation, which sacrifices a small amount of early-stage performance but makes the overall process more privacy-preserving and communication-efficient.

\begin{figure}[htbp]
    \centering
    \includegraphics[width=0.48\textwidth]{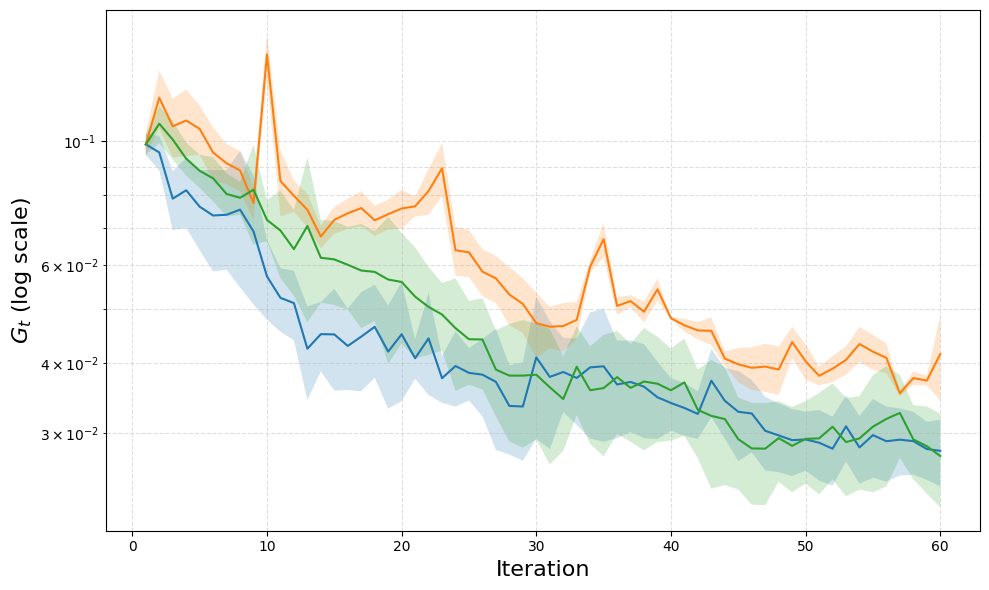}
    \caption{Comparison of regret $G_t$ in heterogeneous settings using \texttt{levy 1-3} function. 
    Each curve shows the mean regret with shaded 95\% confidence intervals. 
    \textbf{Methods}: 
    \solidline{Random}{1.5em} \texttt{CCBO} with RFF approximation, 
    \solidline{Independent}{1.5em} \texttt{CCBO},
    \solidline{orange}{1.5em} {MTS}.}
    \label{rffplot}
\end{figure}

\section{Application to Hot Rolling}\label{sec:app}
We consider a case study on optimizing grain size ($Z$) in a hot rolling process, a key indicator of material quality. The study is conducted using a physics-informed simulation framework that couples a thermal-mechanical finite element model with a grain evolution model to capture the dependence of grain size on processing conditions. 

\begin{figure}[htbp]
\vspace{-1em}
\centering
\includegraphics[width=0.48\textwidth]{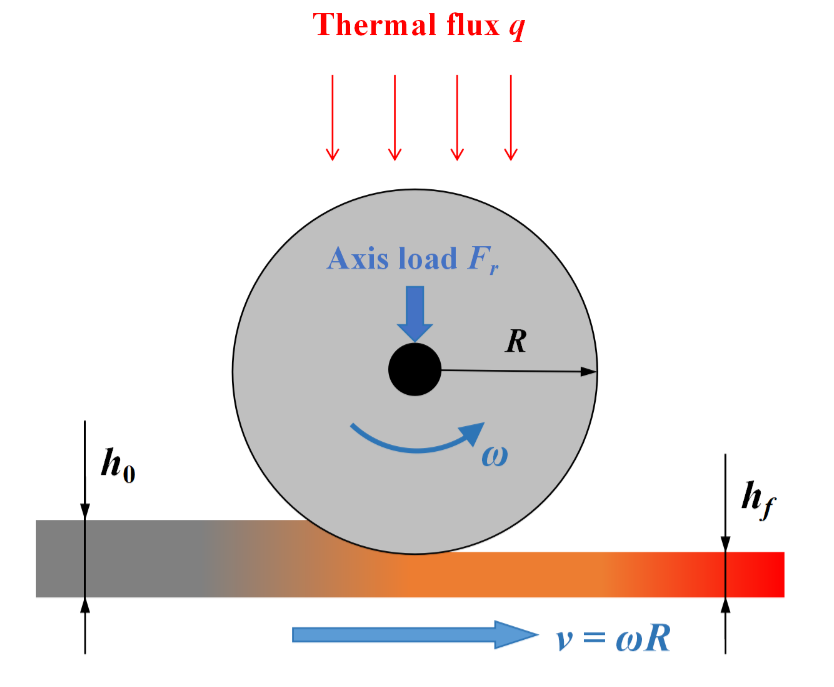}
\caption{Schematic of the roller-based hot rolling process simulated in COMSOL Multiphysics. The roller is characterized by its radius $R$, angular velocity $\omega$, and axial load $F_r$. The blank undergoes thickness reduction from initial thickness $h_0$ to final thickness $h_f$, while subjected to boundary thermal flux $q$ applied on the roller surface. The roller--blank interface exhibits coupled mechanical and thermal interactions, driving localized temperature rise and plastic deformation. The exit velocity of the workpiece is approximated by $V = \omega R$.}
\label{fig:rolling}
\end{figure}

As illustrated in Fig.~\ref{fig:rolling}, a detailed roller-sheet interaction model is constructed in COMSOL Multiphysics. These parameters jointly govern the deformation and thermal behavior during rolling. The parameter ranges are summarized in Table~\ref{tab:par}.

\begin{table}[ht]
\centering 
\caption{Range of each input parameter} 
\label{tab:par} 
\begin{tabular}{|l|c|c|} 
\hline 
\textbf{Parameter name} & \textbf{Min. value} & \textbf{Max. value}\\ 
\hline $R$ (m) & 0.4 & 0.6 \\ 
$\omega$ (rad/s) & 0.1 & 2.5  \\ 
$F_r$ ($\times 10^6$ N) & 0.1 & 1.0  \\ 
$q$ ($\times 10^6$ W/m$^2$) & 30 & 70 \\
$h_f$ (m) & 0.02 & 0.03  \\ \hline 
\end{tabular} 
\end{table}

Using COMSOL’s multiphysics solver, we simulate the transient temperature distribution across the sheet during the rolling process. We extract the average temperature in the mid-to-exit region of the contact zone, denoted $T_{\text{ave}}$, which serves as a critical input to the grain evolution model. This temperature reflects the balance among mechanical work input, heat conduction, interfacial heat transfer, and external boundary fluxes, and directly affects recrystallization and grain growth kinetics. The evolution of grain size is then modeled in MATLAB using a physically informed formulation that accounts for both deformation-induced recrystallization and thermally activated grain growth. The simulator maps inputs $(\omega, F_r, q, h_f, R)$ to grain size $Z(\omega, F_r, q, h_f, R)$. Further details of the grain size simulation are provided in Appendix~\ref{appreal}.

\begin{figure}[htbp]
\centering
\includegraphics[width=0.48\textwidth]{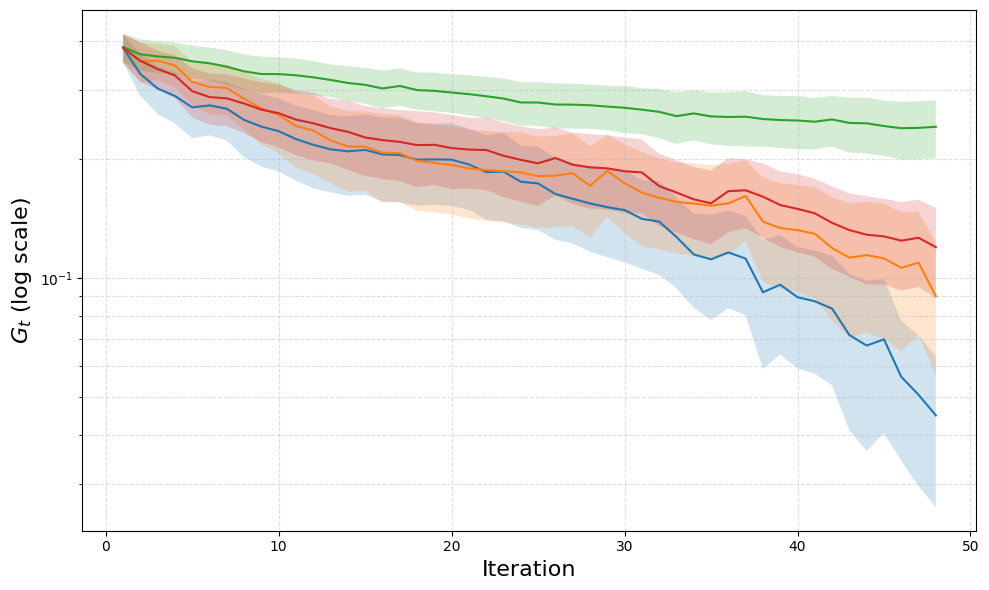}
\caption{Comparison of regret $G_t$ in heterogeneous settings in the hot rolling experiment. Each curve shows the mean regret with shaded 95\% confidence intervals. 
    \textbf{Methods}: 
    \solidline{Federated}{1.5em} FTS, 
    \solidline{Random}{1.5em} RS, 
    \solidline{orange}{1.5em} {MTS},
    \solidline{Independent}{1.5em} \texttt{CCBO}.}
\label{fig:rollingres}
\end{figure}

We aim to characterize how the optimal process parameters vary with the final sheet thickness. Here, the thickness $h_f$ defines the context, while variability in the roller radius $R$ induces client-level heterogeneity. We generate $K$ clients by sampling $\{R_k\}_{k=1}^K \sim \mathcal{N}(0.5, 0.1)$ and treating each $R_k$ as fixed. The problem is thus formulated as a controllable-context BO (CBO) task with design variable $x=(\omega, F_r, q)$, context $c=h_f$, and response function $f_k(x,c)=Z(\omega,F_r,q,h_f,R_k)$. The objective is to learn the context-specific optimal design
\[
x_k^*(c) = \arg\max_{\omega, F_r, q} Z(\omega, F_r, q, h_f, R_k).
\]
To improve computational efficiency, we generate $10{,}000$ samples from the simulator and train a neural network (NN), which is used as a surrogate oracle during optimization. As in Section~\ref{sec:sim}, observations are perturbed with Gaussian noise, where the noise level is scaled relative to the empirical variability of the simulator outputs. All input parameters are normalized to the unit cube to improve numerical stability when fitting the $\mathcal{GP}$.

We consider $K=10$ heterogeneous clients. The number of iterations is set to $T=50$, and each experiment is repeated $30$ times. All other settings, including regret evaluation, the schedule of $p_t$, and initial data generation, remain consistent with those in Section~\ref{sec:sim}. Fig.~\ref{fig:rollingres} illustrates the regret calculated using $G_t$.

It is evident that random sampling fails to approach the optimal design. While FTS leverages information from other clients, its performance remains comparable to MTS. In contrast, \texttt{CCBO} reduces regret more effectively in the early iterations and ultimately achieves the lowest regret among all methods. This demonstrates the effectiveness of the proposed approach in a realistic manufacturing setting and highlights the benefit of collaboration in learning context-specific optimal process parameters.

\section{Discussion and Conclusion}\label{sec:con}

We studied the problem of learning context-specific optimal designs and proposed a collaborative contextual Bayesian optimization framework that enables multiple clients to jointly approximate the mapping $x^*(c)$. By combining cross-client information sharing with adaptive switching between collaborative and independent decisions, the approach improves sample efficiency while preserving the ability to specialize to client-specific behavior. We establish theoretical guarantees with sublinear regret and demonstrate strong empirical performance across benchmark functions and a realistic hot rolling application.

Several directions remain open for future work. A deeper understanding of optimal regret rates in contextual settings would provide useful benchmarks for algorithm design. In addition, developing more adaptive and robust collaboration mechanisms, particularly under strong heterogeneity or limited communication, is an important direction for improving scalability. Overall, this work highlights the potential of collaborative approaches for enabling efficient and scalable contextual optimization in complex real-world systems.

\section*{Acknowledgments}
This work was supported by the National Science Foundation under Grant CMMI-2328010.

\bibliography{Ref}
\bibliographystyle{icml2026}

%%%%%%%%%%%%%%%%%%%%%%%%%%%%%%%%%%%%%%%%%%%%%%%%%%%%%%%%%%%%%%%%%%%%%%%%%%%%%%%
%%%%%%%%%%%%%%%%%%%%%%%%%%%%%%%%%%%%%%%%%%%%%%%%%%%%%%%%%%%%%%%%%%%%%%%%%%%%%%%
% APPENDIX
%%%%%%%%%%%%%%%%%%%%%%%%%%%%%%%%%%%%%%%%%%%%%%%%%%%%%%%%%%%%%%%%%%%%%%%%%%%%%%%
%%%%%%%%%%%%%%%%%%%%%%%%%%%%%%%%%%%%%%%%%%%%%%%%%%%%%%%%%%%%%%%%%%%%%%%%%%%%%%%

\newpage
\appendix
\onecolumn
\section{Technical Results} 
\subsection{Proof for Theorem~\ref{regbd}}\label{app:pf1}
Let $X_t \sim \text{Bernoulli}(p_t)$ denote the switching variable at 
iteration $t$, where $X_t = 1$ indicates that the collaborative scheme is selected 
and $X_t = 0$ indicates the independent scheme. We assume that $\{X_t\}_{t \geq 1}$ 
are independent across $t$, and that each $X_t$ is independent of the observation 
noise $\{\epsilon_{k,t}\}$ and of the posterior updates $\{\mathcal{GP}(\mathcal{D}_{k,t})\}$. 
This independence holds by construction, since the Bernoulli gate is drawn before 
the design--context pair is selected and before the response is observed.
\begin{lemma}[Expected last occurrence time]
\label{lemma:upperN}
Let $\{X_t\}_{t\ge 1}$ be a sequence of independent 
$\{0,1\}$-valued random variables with $\mathbb{P}(X_t = 1) = p_t$ for each $t \ge 1$, where $X_t = 1$ 
indicates that the collaborative scheme is used at time $t$, and $X_t=0$ otherwise. Define the random variable
\[
N:=\sup\{t\ge 1:X_t=1\},
\]
with the convention $N=0$ if $X_t=0$ for all $t\ge 1$. If $\sum_{t=1}^\infty t\,p_t<\infty,$ then $N<\infty$ almost surely and $\mathbb{E}[N]\le \sum_{t=1}^\infty t\,p_t.$
\end{lemma}

\begin{proof}
The proof is divided into two steps. First, we show that $N<\infty$ almost surely (to show that $\mathbb{E}[N]$ exists). Second, we derive an upper bound for $\mathbb{E}[N]$. Observe that
\[
\sum_{t=1}^\infty \mathbb{P}(X_t=1)=\sum_{t=1}^\infty p_t
\le \sum_{t=1}^\infty t\,p_t<\infty.
\]
Since this series converges, the first Borel--Cantelli lemma \cite{feller1971introduction} implies that
\[
\mathbb{P}\bigl(\limsup_{t\to\infty}\{X_t=1\}\bigr)=0.
\]
Equivalently, with probability one, only finitely many indices $t$ satisfy $X_t=1$. Therefore, $N=\sup\{t\ge 1:X_t=1\}<\infty$ almost surely. 

Next, since $N$ is a nonnegative integer-valued random variable, the tail-sum formula gives
\[
\mathbb{E}[N]=\sum_{n=1}^\infty \mathbb{P}(N\ge n).
\]
For each $n\ge 1$, the event $\{N\ge n\}$ means that there exists some time $s\ge n$ such that $X_s=1$. Thus
\[
\{N\ge n\}=\bigcup_{s=n}^\infty \{X_s=1\}.
\]
Applying the union bound, we obtain
\[
\mathbb{P}(N\ge n)
=\mathbb{P}\Bigl(\bigcup_{s=n}^\infty \{X_s=1\}\Bigr)
\le \sum_{s=n}^\infty \mathbb{P}(X_s=1)
=\sum_{s=n}^\infty p_s.
\]
Substituting this into the tail-sum formula yields
\[
\mathbb{E}[N]\le \sum_{n=1}^\infty \sum_{s=n}^\infty p_s
=\sum_{s=1}^\infty \sum_{n=1}^s p_s
=\sum_{s=1}^\infty s\,p_s.
\]
The interchange of summation order is justified since 
all terms $p_s \ge 0$. Explicitly,
$\sum_{s=1}^\infty \left(\sum_{n=1}^s 1\right) p_s
= \sum_{s=1}^\infty s\, p_s$.
\end{proof}

Lemma ~\ref{lemma:upperN} shows that if the probability of using the collaborative scheme decreases quickly enough as time increases, then the scheme will eventually stop being used almost surely. Moreover, the expected time of its final use is finite, and can be bounded by $\sum_{t=1}^\infty t\,p_t$.

The following lemma provides an upper bound on the regret of MTS.

\begin{lemma}\label{lemma1} (Theorem~1 in \cite{char2019offline})
Let $\gamma_t=\max_{k}\max_{|\mathcal{D}_{k,t}|=t}I\!\left(\mathbf{y}_{k}; \mathbf{f}_k \mid \mathcal{D}_{k,t}\right)$ denote the maximum information gain at time $t$. Under Assumptions~\ref{assump1}, the expected instantaneous regret at time $t$ admits the following upper bound:
\[\label{eq:lemma1}
\mathbb{E}\!\left[r_t\right]
\;\leq\;
|\mathcal{C}|\!\left(\frac{1}{t}
+
\sqrt{\frac{|\mathcal{X}||\mathcal{C}|\gamma_t}{2t}}\right).
\]
\end{lemma}

Lemma~\ref{lemma1} provides the key ingredient 
for bounding the regret once collaboration ceases. 
Combined with Lemma~\ref{lemma:upperN}, which bounds 
the cost of the collaborative phase, we can now establish the full regret bound for \texttt{CCBO}.

\begin{proof}[Proof of Theorem~\ref{regbd}]
We condition on the realization of the switching sequence $\{X_t\}_{t=1}^T$, which is independent of the stochastic observations $\{y_{k,t}\}$ and posterior updates by the independence assumption stated above. Let $N$ be the last time 
the collaborative scheme is invoked, as defined in Lemma~\ref{lemma:upperN}. By 
Lemma~\ref{lemma:upperN}, $N < \infty$ almost surely under 
the assumption $\sum_{t=1}^\infty t p_t < \infty$. We decompose the regret into two parts: the regret accumulated up to the last collaborative round, and the regret accumulated afterwards. More precisely,
\[
R_T
=
\sum_{t=1}^{N\wedge T} r_t
+
\sum_{t=N\wedge T+1}^{T} r_t.
\]
Taking expectations gives
\[
\mathbb{E}[R_T]
=
\mathbb{E}\!\left[\sum_{t=1}^{N\wedge T} r_t\right]
+
\mathbb{E}\!\left[\sum_{t=N\wedge T+1}^{T} r_t\right].
\]

We first bound the regret incurred before time $N$. By definition of $r_t$ in~\eqref{eq:disc_regret}, for 
each context $c \in \mathcal{C}$ and any ${x}^{\text{best}}_{k,t}(c)$,
\[
f_k\bigl(x_k^{\text{best}}(c),c\bigr) - f_k\bigl({x}^{\text{best}}_{k,t}(c),c\bigr)
\;\le\;
\max_{x_1,x_2 \in \mathcal{X}}
\bigl\{f_k(x_1,c) - f_k(x_2,c)\bigr\}.
\]
Summing over $c \in \mathcal{C}$ and dividing by the 
denominator of~\eqref{eq:disc_regret} gives $r_t \le 1$ almost surely for all $t \ge 1$. Consequently,
\[
\sum_{t=1}^{N\wedge T} r_t \le N\wedge T \le N,
\]
and taking expectations gives $\mathbb{E}\bigl[\sum_{t=1}^{N\wedge T} r_t\bigr] \le \mathbb{E}[N]$.

Next, conditional on $\{N = n\}$, no collaborative updates occur after time $n$. That is, the switching variables satisfy $X_t = 0$ for all $t > n$, and the algorithm therefore reduces to MTS from iteration $n + 1$ onward. Under this reduction, the regret incurred after time $n$ is governed solely by the independent Thompson sampling scheme, and we may apply Lemma 2 directly. Therefore, by Lemma~\ref{lemma1},
\[
\mathbb{E}\!\left[\sum_{t=n+1}^{T} r_t \,\middle|\, N=n\right]
\le
\sum_{s=1}^{T-n}
|\mathcal{C}|
\left(
\frac{1}{s}
+
\sqrt{\frac{|\mathcal{X}||\mathcal{C}|\gamma_s}{2s}}
\right),
\]
where we re-index $s = t - n$ so that $s = 1$ corresponds to the first MTS iteration at time $n+1$, and $\gamma_s$ denotes the maximum information gain after $s$ observations in the MTS phase.

Since the maximum information gain $\gamma_t$ is 
non-decreasing in $t$ (i.e., adding more observations can only 
increase the information gain), we have $\gamma_s \le 
\gamma_T$ for all $s \le T$. Therefore,

\[
\mathbb{E}\!\left[\sum_{t=n+1}^{T} r_t \,\middle|\, N=n\right]
\le
\sum_{s=1}^{T}
|\mathcal{C}|
\left(
\frac{1}{s}
+
\sqrt{\frac{|\mathcal{X}||\mathcal{C}|\gamma_s}{2s}}
\right).
\]
Taking expectation over $N$, we obtain
\[
\mathbb{E}\!\left[\sum_{t=N\wedge T+1}^{T} r_t\right]
\le
\sum_{s=1}^{T}
|\mathcal{C}|
\left(
\frac{1}{s}
+
\sqrt{\frac{|\mathcal{X}||\mathcal{C}|\gamma_s}{2s}}
\right).
\]
Combining the two parts yields
\begin{align*}
\mathbb{E}[R_T]&\le\mathbb{E}[N]+\sum_{s=1}^{T}|\mathcal{C}|\left(\frac{1}{s}+
\sqrt{\frac{|\mathcal{X}||\mathcal{C}|\gamma_s}{2s}}
\right)\\
&\le
\sum_{t=1}^\infty t\,p_t
+
\sum_{s=1}^{T}
|\mathcal{C}|
\left(
\frac{1}{s}
+
\sqrt{\frac{|\mathcal{X}||\mathcal{C}|\gamma_s}{2s}}
\right)\qquad\text{(by Lemma~\ref{lemma:upperN})}\\
&\le \mathcal{O}(1)
+
|\mathcal{C}|\sum_{s=1}^{T}\frac{1}{s}
+
|\mathcal{C}|\sum_{s=1}^{T}
\sqrt{\frac{|\mathcal{X}||\mathcal{C}|\gamma_s}{2s}}\qquad\text{($\sum_{t=1}^\infty t\,p_t<\infty$)}\\
&\le\mathcal{O}(\log T)
+
\sqrt{\gamma_T}|\mathcal{C}|\sum_{s=1}^{T}
\sqrt{\frac{|\mathcal{X}||\mathcal{C}|}{2s}}\qquad\text{($\gamma_t$ is non-decreasing)}\\
&\le\mathcal{O}(\log T)
+\mathcal{O}(\sqrt{\gamma_T T}).
\end{align*}
\end{proof}

\subsection{Proof for Theorem~\ref{sharebd}}
\label{app:thm2}
We first introduce the Azuma--Hoeffding Inequality, which 
can be used to control the deviation of a sum of bounded 
independent random variables from its mean.

\begin{lemma}[Azuma--Hoeffding Inequality]
\label{lem:azuma}
Let $(\mathcal{F}_t)_{t\ge 0}$ be a filtration and let 
$(M_t)_{t=0}^T$ be a martingale with respect to 
$(\mathcal{F}_t)$ such that the increments are almost 
surely bounded:
\[
|M_t - M_{t-1}| \le c_t \quad \text{a.s.} 
\qquad (t=1,\dots,T).
\]
Then, for any $\epsilon>0$,
\[
\Pr\!\big(M_T - M_0 \ge \epsilon\big) 
\;\le\; 
\exp\!\left(
  -\frac{\epsilon^2}{2\sum_{t=1}^T c_t^2}
\right).
\]
\end{lemma}

\begin{proof}[Proof of Theorem~\ref{sharebd}]

\textit{Case $K=1$.}
Denote by $I_T^{(1)} = \sum_{t=1}^T Q_t$ the total 
number of collaborative rounds for a single client up 
to time $T$, where $Q_t \in \{0,1\}$ indicates whether 
posterior mean sharing occurs at iteration $t$, with 
$\Pr(Q_t = 1) = p_t$. Note that $p_t$ is the per-client 
probability of triggering collaboration at iteration $t$, 
and each client makes this decision independently of all 
other clients and of past decisions. We have 
$\mathbb{E}[Q_t] = p_t$.

Define $Y_t \;=\; \sum_{s=1}^t \big(Q_s - p_s\big), 
\qquad Y_0 = 0.$
Since the Bernoulli gate at iteration $t$ is drawn 
independently of all past decisions, $Q_t$ is independent 
of $\mathcal{F}_{t-1}$, and therefore
\[
\mathbb{E}[Y_t \mid \mathcal{F}_{t-1}]= Y_{t-1} + \mathbb{E}[Q_t - p_t \mid \mathcal{F}_{t-1}]= Y_{t-1} + \mathbb{E}[Q_t] - p_t= Y_{t-1},
\]

so $Y_t$ is a martingale with respect to the natural 
filtration, and $\mathbb{E}[Y_t] = 0$ for all $t \ge 1$. 
The increments are bounded almost surely:
\[
|Y_t - Y_{t-1}| \;=\; |Q_t - p_t| \;\le\; 1 
\quad \text{a.s.}
\]
Applying Lemma~\ref{lem:azuma} with $c_t = 1$ for all 
$t$ and setting $\epsilon = \sqrt{2T\log(1/\delta)}$ gives
\[
\Pr\!\Big(Y_T \ge \sqrt{2T\log(1/\delta)}\Big)
\;\le\; 
\exp\!\left(-\frac{2T\log(1/\delta)}{2T}\right)
= \delta.
\]
Hence, with probability at least $1 - \delta$,
\[
Y_T \;\le\; \sqrt{2T\log(1/\delta)}.
\]
Since $I_T^{(1)} = \sum_{t=1}^T p_t + Y_T$, we obtain 
with probability at least $1 - \delta$:
\[
I_T^{(1)} 
\;\le\; 
\sum_{t=1}^T p_t + \sqrt{2T\log(1/\delta)}
\;=\; 
\mathcal{O}\!\left(\sqrt{T\log(1/\delta)}\right),
\]
where the last step uses $\sum_{t=1}^T p_t = 
\mathcal{O}(\sqrt{T})$ by the assumption 
$\sum_{t=1}^T p_t = \mathcal{O}(\sqrt{T})$ in 
Theorem~\ref{sharebd}.

\textit{General $K \ge 1$, without a central server.}
At each collaborative iteration, each client broadcasts 
a $D$-dimensional weight vector to all other $K-1$ 
clients, incurring a per-round communication cost of 
$\mathcal{O}(DK)$. Since the total number of 
collaborative rounds per client is $I_T^{(1)} = 
\mathcal{O}(\sqrt{T\log(1/\delta)})$ with probability 
at least $1-\delta$ (from the $K=1$ case above), the 
total communication cost satisfies
\[
I_T \;=\; 
\mathcal{O}\!\left(DK\sqrt{T\log(1/\delta)}\right).
\]

\textit{General $K \ge 1$, with a central server.}
With a central server, each client sends only a 
$D$-dimensional weight vector to the server per 
collaborative round, incurring a per-round cost of 
$\mathcal{O}(D)$ per client. The total communication 
cost is therefore $\mathcal{O}(D \cdot I_T^{(1)}) = 
\mathcal{O}(D\sqrt{T\log(1/\delta)})$. The 
$\mathcal{O}(DT)$ term in the minimum follows from the 
trivial upper bound that there are at most $T$ total 
iterations, each incurring at most $\mathcal{O}(D)$ 
communication cost per client to the server. Combining both bounds gives
\[
I_T \;=\; 
\mathcal{O}\!\left(
  \min\!\left\{
    DK\sqrt{T\log(1/\delta)},\; DT
  \right\}
\right).
\]
Substituting $I_T^{(1)} = 
\mathcal{O}(\sqrt{T\log(1/\delta)})$ into both 
architectures recovers exactly the bounds stated in 
Theorem~\ref{sharebd}.
\end{proof}

\section{Details for Simulation Studies}\label{appsim}
\label{app:benchmarks}

We adopt three standard continuous benchmark functions, \texttt{hartmann}, \texttt{levy}, and \texttt{ackley}, to evaluate the performance of our contextual Bayesian optimization framework.  
Let the input be decomposed into a \emph{context} vector $c \in [0,1]^{D_c}$ and a \emph{design} vector $x \in [0,1]^{D_x}$, with total dimension $D = D_c + D_x$.  
Throughout all experiments, the input domain is normalized as
\[
u = (c, x) \in [0,1]^{D_c + D_x}.
\]

\subsection{\texttt{hartmann} Function}
The six-dimensional \texttt{hartmann} function is defined on $[0,1]^6$ as
\[
f_{\texttt{hartmann}}(v)
= - \sum_{i=1}^{4} \alpha_i
    \exp\!\left(
    -\sum_{j=1}^{6} A_{ij}(v_j - P_{ij})^2
    \right),
\]
where
\[
\alpha = (1.0,\, 1.2,\, 3.0,\, 3.2),
\]
\begin{align*}
A &=
\begin{bmatrix}
10 & 3 & 17 & 3.5 & 1.7 & 8 \\
0.05 & 10 & 17 & 0.1 & 8 & 14 \\
3 & 3.5 & 1.7 & 10 & 17 & 8 \\
17 & 8 & 0.05 & 10 & 0.1 & 14
\end{bmatrix}\\P &= 10^{-4}\!\times
\begin{bmatrix}
1312 & 1696 & 5569 & 124 & 8283 & 5886 \\
2329 & 4135 & 8307 & 3736 & 1004 & 9991 \\
2348 & 1451 & 3522 & 2883 & 3047 & 6650 \\
4047 & 8828 & 8732 & 5743 & 1091 & 381
\end{bmatrix}.
\end{align*}

In our setting, we use $D_c = 2$ and $D_x = 2$.  
To embed the 4-dimensional contextual–design pair $(c, x)$ into the 6-dimensional \texttt{hartmann} input $v \in [0,1]^6$, we define
\[
v = (c_1,\, c_2,\, x_1,\, x_2,\, 0.5,\, 0.5),
\]
where the last two coordinates are fixed at $0.5$ to maintain the same dimension as the original \texttt{hartmann}-6 definition.  
This preserves the nonlinear coupling of the original function while allowing controlled contextual variations.

\subsection{\texttt{levy} Function}
The \texttt{levy} function is originally defined on $[-10, 10]^D$ and is rescaled to $u_i \in [0,1]$ by
\[
z_i = -10 + 20u_i, \qquad w_i = 1 + \frac{z_i - 1}{4}.
\]
The resulting form is
\[
f_{\texttt{levy}}(u)= -\sin^2(\pi w_1)- \sum_{i=1}^{D-1} (w_i - 1)^2 \!\left[1 + 10 \sin^2(\pi w_i + 1)\right]- (w_D - 1)^2 \!\left[1 + \sin^2(2\pi w_D)\right].
\]

The global maximum occurs at $u_i = 0.5$ (corresponding to $z_i = 0$) with $f_{\texttt{levy}}(u^*) = 0$.

\subsection{\texttt{ackley} Function}
The \texttt{ackley} function is defined on $[-32.768, 32.768]^D$.  
After rescaling each $u_i \in [0,1]$ to
\[
z_i = -32.768 + 65.536u_i,
\]
it becomes
\[
f_{\texttt{ackley}}(u)= a \exp\!\left(-b \sqrt{\frac{1}{D}\sum_{i=1}^{D} z_i^2}\right)+\exp\!\left(\frac{1}{D}\sum_{i=1}^{D}\cos(c z_i)\right)
 - a - e,
\]
where $a = 20$, $b = 0.2$, and $c = 2\pi$.  
The global maximum occurs at $u_i = 0.5$ with $f_{\texttt{ackley}}(u^*) = 0$.

\section{Details for Simulating Grain Size in Hot Rolling Process}\label{appreal}

We first compute the equivalent flow stress $\sigma$ in the deformation zone, which is estimated based on the average pressure $p_{\text{avg}}$ obtained from the roller contact mechanics model. Under the assumption of plane strain hot rolling, the average stress is given by $\sigma = \frac{F_r}{b \sqrt{R(h_0 - h_f)}},$ where $b$ is the sheet width.

The strain rate $\dot{\varepsilon}$ is not directly prescribed but instead derived from the Zener–Hollomon parameter $Z$, which captures the thermally activated deformation kinetics: $Z = A \left[\sinh(\alpha \sigma)\right]^n,$ where $A$, $\alpha$, and $n$ are material-dependent constants. By inverting the Zener–Hollomon expression and combining it with the Arrhenius relation $Z = \dot{\varepsilon} \exp\left(\frac{Q}{RT_{\text{ave}}}\right),$
where $Q$ is the activation energy ($2.5\times10^5$~J/mol in this study), we estimate the equivalent strain rate under varying thermal and stress conditions. 

The total true strain is $\varepsilon = \ln\!\left(\frac{h_0}{h_f}\right),$ and together with $\dot{\varepsilon}$, defines the effective deformation time $t_{\text{eff}} = \varepsilon / \dot{\varepsilon}$, which plays a key role in grain growth modeling. The final grain size $d_{\text{final}}$ is modeled by combining two competing mechanisms:

\paragraph{Dynamic Recrystallization (DRX)} which refines grains during deformation: $d_{\text{DRX}} = C \, \varepsilon^{-\gamma} \, \dot{\varepsilon}^{-m},$
where $C$, $\gamma$, and $m$ are material constants ($C=5000$, $\gamma=0.5$, $m=0.2$).

\paragraph{Grain Growth} which increases grain size during and after deformation: $d_{\text{GG}} = K_g \, t_{\text{eff}}^{1/n_g} \exp\!\left(-\frac{Q_g}{RT_{\text{ave}}}\right) f(\sigma),$ where $K_g = 1\times10^{-6}$, $Q_g = 2.5\times10^5$, $n_g = 2$, and $f(\sigma)$ represents the stress-sensitive growth suppression term, typically defined as $f(\sigma) = \exp(-\beta \sigma).$ The final grain size is computed as a generalized norm: $d_{\text{final}} = \left( d_{\text{DRX}}^{-p} + d_{\text{GG}}^{-p} \right)^{-1/p},$ which enables a smooth transition between recrystallization-dominated and growth-dominated regimes, depending on strain rate and temperature.

\vspace{0.3cm}
To discourage physically infeasible or economically undesirable process parameters (e.g., excessively high rolling force or large roller radius), we introduce penalty functions into the grain size formulation. Let $R_{\min}$, $R_{\max}$, $F_{r,\min}$, and $F_{r,\max}$ denote the bounds of sampled process parameters. The reference values and scaling factors are defined as:
\[
R_{\text{ref}} = \frac{R_{\max}+R_{\min}}{2}, \quad s_R = R_{\max}-R_{\min},
\]
\[
F_{r,\text{ref}} = \frac{F_{r,\max}+F_{r,\min}}{2}, \quad s_{F_r} = F_{r,\max}-F_{r,\min}.
\]

The penalty functions are then given by:
\[
P_{F_r} = \lambda_F \left(\frac{F_r - F_{r,\text{ref}}}{s_{F_r}}\right)^2, \qquad
P_R = \lambda_R \max\left(0, \frac{R - R_{\text{ref}}}{s_R}\right)^2,
\]
and the total penalty is $P_{\text{total}} = P_{F_r} + P_R.$ Finally, the penalized final grain size is expressed as $d_{\text{final}}' = d_{\text{final}} \, (1 + P_{\text{total}}),$ which ensures that the optimization framework balances metallurgical performance with realistic process conditions.

\end{document}